\documentclass{article}

\usepackage[preprint]{neurips_2026}

\usepackage[utf8]{inputenc} 
\usepackage[T1]{fontenc}    
\usepackage{hyperref}       
\usepackage{url}            
\usepackage{booktabs}       
\usepackage{amsfonts}       
\usepackage{nicefrac}       
\usepackage{microtype}      
\usepackage{xcolor}         
\usepackage{booktabs}
\usepackage{multirow}

\usepackage{amsmath}
\usepackage{amsthm}
\usepackage{amssymb}
\usepackage{xcolor}
\usepackage{graphicx} 
\usepackage{float}
\usepackage{wrapfig}
\usepackage{graphicx}
\usepackage{array}

\usepackage{algorithm}
\usepackage{algorithmic}

\newtheorem{theorem}{Theorem}
\newtheorem{proposition}{Proposition}

\newcommand{\LNRE}{\gL_{\text{NRE}}}
\newcommand{\LKLIEP}{\gL_{\text{KLIEP}}}
\newcommand{\LDSM}{\gL_{\text{DSM}}}
\newcommand{\LSBI}{\gL_{\text{SBI}}}

\usepackage{amsmath,amsfonts,bm}

\def\eqref#1{equation~\ref{#1}}

\def\1{\bm{1}}

\def\rvtheta{{\boldsymbol{\theta}}} 

\def\rvc{{\mathbf{c}}}

\def\rvx{{\mathbf{x}}}

\def\rvz{{\mathbf{z}}}

\def\rmI{{\mathbf{I}}}

\def\evtheta{{\theta}}

\DeclareMathAlphabet{\mathsfit}{\encodingdefault}{\sfdefault}{m}{sl}
\SetMathAlphabet{\mathsfit}{bold}{\encodingdefault}{\sfdefault}{bx}{n}

\def\gD{{\mathcal{D}}}

\def\gL{{\mathcal{L}}}

\def\gN{{\mathcal{N}}}

\def\gU{{\mathcal{U}}}

\def\gX{{\mathcal{X}}}

\def\gCN{{\mathcal{CN}}} 

\newcommand{\E}{\mathbb{E}}

\newcommand{\R}{\mathbb{R}}

\title{Guidance for Prior Change via Density Ratio Estimation}

\author{
Yichen Zang \\
Department of Mathematics\\
University of Bristol\\
Bristol, UK \\
\texttt{yichen.zang@bristol.ac.uk}
\And
Jiun-Yi Lin \\
Department of Mathematics\\
University of Bristol\\
Bristol, UK \\
\texttt{am21541@bristol.ac.uk}
\And
Song Liu \\
Department of Mathematics\\
University of Bristol\\
Bristol, UK \\
\texttt{song.liu@bristol.ac.uk}
}

\begin{document}

\maketitle

\begin{abstract}
Simulation-Based Inference (SBI) serves as a vital framework for parameter inference in scientific fields where simulators involve intractable likelihoods, yet while amortized generative models offer rapid posterior estimation, they are often restricted by the specific priors used during training, thereby limiting their flexibility as prior knowledge evolves. To address this prior dependency, PriorGuide was introduced as an inference-time guidance method, but due to its intractable formulation, it relies on Gaussian approximations of the reverse transition kernel and Gaussian mixture model fitting for the prior ratio, both of which introduce systematic bias. Motivated by these limitations, we propose an unbiased test-time guidance framework that leverages Density Ratio Estimation (DRE) to learn a score guidance term, effectively  decoupling the inference process from the prior training. Moreover, our framework remains agnostic to the specific density ratio estimators, making it a general and flexible framework for handling prior changes. 
Experimental results across multiple tasks demonstrate that our method matches or outperforms PriorGuide on C2ST and MMD in most tasks while maintaining robustness even under limited overlap between the training and target priors. Furthermore, we apply our method to Bayesian updating for parameter inference from planetary light-curve data, where it also demonstrates strong effectiveness and robustness. Code is available at \url{https://github.com/a-chenchen/dre-based-prior-guidance}.
\end{abstract}

\section{Introduction}
\label{sec:Introduction}
In the data-driven era, Simulation-Based Inference (SBI) is a fundamental framework for parameter inference in complex systems, such as astrophysics \citep{ho_2024_ltuili}, particle physics \citep{brehmer_2021_simulationbased}, and epidemiology \citep{li_2025_advances}. Here, while simulators easily generate synthetic data $\rvx$ for parameters $\rvtheta$, the likelihood $p(\rvx \mid \rvtheta)$ is analytically intractable, rendering traditional Bayesian methods like MCMC inapplicable.

While Approximate Bayesian Computation (ABC) has long been the standard likelihood-free approach, its reliance on manually crafted summary statistics often leads to significant information loss, and its computational efficiency scales poorly with the dimensionality of the parameter space. To overcome these limitations, recent advancements in deep generative modeling have transformed the landscape of SBI. Modern approaches, such as Normalizing Flows \citep{papamakarios_2021_normalizing}, Diffusion Models \citep{song_2021_scorebased}, and Flow Matching \citep{liu_2022_flow}, enable the development of amortized inference methods \citep{sharrock_2022_sequential, dax_2023_flow}. By training on paired datasets $\{(\rvtheta^{(i)}, \rvx^{(i)})\}_{i=1}^N$ sampled from the joint distribution $p(\rvtheta, \rvx)$, these models can provide high-quality posterior estimates for new observations without additional simulations, enabling efficient real-time data analysis.

Despite its efficiency, a major limitation of this amortized paradigm is its rigid dependency on the training prior $p(\rvtheta)$. In standard Neural Posterior Estimation (NPE), the prior distribution is implicitly embedded in the model's weights during training. Therefore, even though these models generalize well across different observations $\rvx$, they remain tied to their training prior. This inherent rigidity hinders iterative scientific inquiry, where prior knowledge is frequently refined in light of new data or theoretical insights. Such requirements for dynamic prior adaptation are prevalent across diverse disciplines, including nuclear physics \citep{jiang_2022_bayesian} and engineering \citep{liu_2023_efficient}.

\begin{wrapfigure}{r}{0.5\textwidth}
    \centering
    \includegraphics[width=\linewidth]{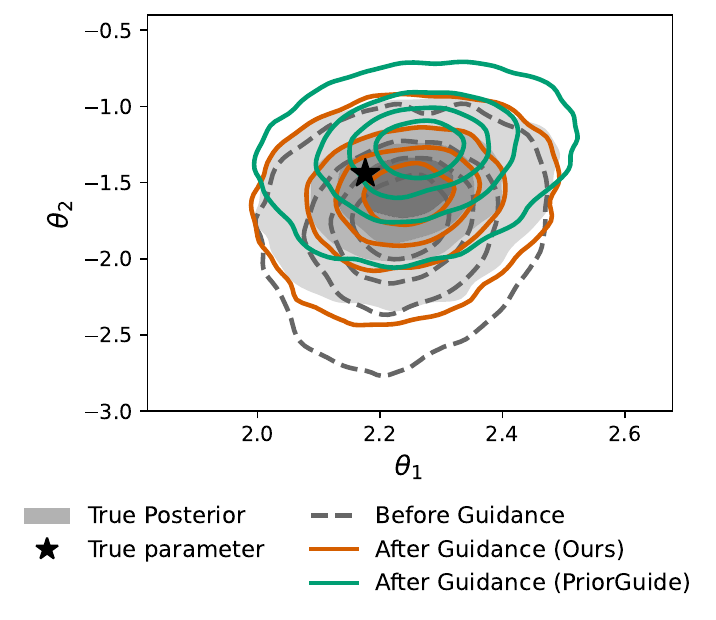}
    \caption{Two-dimensional Posterior approximations under prior shift for Ornstein-Uhlenbeck Process. Grey shading denotes the true posterior; dashed contour the base model; colored contours our method and PriorGuide after guidance. Our method closely tracks the truth while PriorGuide exhibits deviation.}
    \label{fig:bias}
\end{wrapfigure}

When the prior changes to $q(\rvtheta)$, the pre-trained amortized SBI model becomes biased, as it continues to reflect the now-obsolete training prior $p(\rvtheta)$. The classical solution is to re-simulate the entire dataset from $q(\rvtheta)p(\rvx\mid \rvtheta)$ and retrain an SBI model. However, especially for complex simulators, this process has a very high computational cost, making it prohibitively expensive. Recent research has begun to explore "prior amortization" strategies to handle varying priors \citep{chang_2024_amortized, whittle_2025_distribution}. These methods typically require pre-defining a specific family of priors or discretizing the prior space during training, which fails to accommodate truly arbitrary prior changes encountered at test-time.

A more efficient way to handle prior changes is test-time guidance, such as the recently proposed PriorGuide \citep{yang_2025_priorguide}. Given a diffusion-based SBI model pre-trained on $p(\rvtheta)$ and prior ratio between the old and new prior $\frac{q(\rvtheta)}{p(\rvtheta)}$, it introduces a guidance term to the base diffusion model during testing, to adjust the generated samples to a new posterior. However, PriorGuide relies on two approximations due to its intractable formulation: a GMM-based fit for the prior ratio and a Gaussian approximation for the reverse transition kernel. These approximations may introduce significant bias into the estimation of the true guidance, as illustrated in Figure~\ref{fig:bias}.

In this paper, we propose a ratio-estimation-based test-time guidance method. Our approach utilizes a lightweight Density Ratio Estimation (DRE) \citep{durkan2020contrastive, hermans2020likelihood}, trained across diffusion steps, to directly learn an unbiased and robust guidance term. Our method offers the following three advantages:

\begin{itemize}
\item \textbf{Flexible and Unbiased Inference:}
Our method estimates the guidance term without the systematic discrepancy incurred by PriorGuide's Gaussian kernel approximation, 
ensuring unbiased adaptation to arbitrary prior distributions. It is agnostic to ratio estimation methods and provides amortized inference across different observations.
\item \textbf{Computational Efficiency:}
Our method enables direct reuse of existing datasets without additional simulator calls, streamlining the training process. 
The guidance term is trained using a lightweight density ratio estimator, 
enabling efficient test-time inference without further correction, which is required by PriorGuide.

\item \textbf{Robustness:} Adding Gaussian noise to samples has been proposed to ensure that distributions are well-supported, making the density ratio well-defined \citep{zhang_2020_spread, roth_2017_stabilizing}. Since our model is trained on noisy diffusion states, this mechanism of diffusion naturally stabilizes the ratio estimation process and improves robustness when the target prior is insufficiently supported by the training prior. In contrast, incorporating such noise-based stabilization into PriorGuide is non-trivial, since its guidance relies only on density ratio over clean samples.

\end{itemize}

\section{Background}
\label{sec:Background}
\subsection{Simulation-based Inference}

The primary objective in many scientific disciplines is to infer model parameters $\rvtheta$ from observed data $\rvx$. In Bayesian framework, the posterior distribution is defined as $p(\rvtheta\mid \rvx) \propto p(\rvx\mid \rvtheta)p(\rvtheta)$, where $p(\rvtheta)$ represents the prior and $p(\rvx\mid \rvtheta)$ is the likelihood. Since the posterior intrinsically depends on the choice of prior, the prior's specification plays a critical role in the inference result.

In many scientific applications, the likelihood $p(\rvx\mid \rvtheta)$ is analytically intractable or computationally prohibitive to evaluate, even though samples can be drawn from it via stochastic simulators, making it difficult to directly solve for $p(\rvtheta\mid \rvx)$. To address this, Simulation-Based Inference (SBI) \citep{deistler2025simulation, cranmer2020frontier} leverages parameter-observation pairs $(\rvtheta, \rvx) \sim p(\rvx\mid \rvtheta)p(\rvtheta)$ generated from the simulator to learn the underlying statistical relationships, typically with machine learning methods.

\subsection{Diffusion Models}
\label{sec:diffusion}

Diffusion models map a complex data distribution $p(\rvz_0)$ onto a tractable Gaussian distribution via a forward perturbation process and recover samples by reversing it \citep{ho2020ddpm, song_2021_scorebased}. Following the Variance Preserving (VP) formulation, the diffused state $\rvz_t$ at time $t \in [0, 1]$ is governed by:
\begin{equation}
p(\rvz_t \mid  \rvz_0) = \gN(\rvz_t; \sqrt{1-\sigma^2(t)}\rvz_0, \sigma^2(t)\rmI),
\label{eq:vp-diff}
\end{equation}
where $\sigma(t)$ denotes a noise schedule increasing from $\sigma(0) = 0$ to a terminal value $\sigma(1) \approx 1$. New samples are synthesized by reversing this trajectory through a reverse-time Stochastic Differential Equation (SDE):
\begin{equation}
d\rvz_t = \left[ -\frac{1}{2}\beta(t)\rvz_t - \beta(t)\nabla_{\rvz_t} \log p(\rvz_t) \right] dt + \sqrt{\beta(t)} d\bar{\boldsymbol{\omega}}_t,
\end{equation}

where $\bar{\boldsymbol{\omega}}_t$ is a standard reverse-time Wiener process and $\beta(t) = -\frac{d}{dt}\log(1-\sigma^2(t))$ is the drift coefficient associated with the noise schedule $\sigma(t)$.
The intractable score function $\nabla_{\rvz_t} \log p(\rvz_t)$ is approximated by a neural network $\mathbf{s}_\phi(\rvz_t, t)$ trained via the denoising score matching objective:
\begin{equation}
\LDSM = \E_{t,\, p(\rvz_0),\, p(\rvz_t \mid \rvz_0)} \left[ \lambda(t) \left\| \mathbf{s}_\phi(\rvz_t, t) - \nabla_{\rvz_t} \log p(\rvz_t\mid \rvz_0) \right\|^2 \right],
\end{equation}
where the conditional score is analytically given by $\nabla_{\rvz_t} \log p(\rvz_t\mid \rvz_0) = \frac{-(\rvz_t - \sqrt{1-\sigma^2(t)}\rvz_0)}{\sigma^2(t)}$, and $\lambda(t)$ is a weighting function.

\subsection{Diffusion-based SBI}
\label{sec:diffusion_sbi}

Diffusion-based SBI \citep{sharrock_2022_sequential} casts posterior estimation as a conditional generative task: a score network, conditioned on the observed
data $\rvx$, learns the score of the diffused posterior, which drives a reverse
diffusion transporting Gaussian noise to samples from $p(\rvtheta \mid \rvx)$.

The diffusion process from Section~\ref{sec:diffusion} is applied to $\rvtheta$,
with $\rvtheta_0 := \rvtheta$ and $\rvtheta_t \mid \rvtheta_0 \sim
\mathcal{N}(\sqrt{1-\sigma^2(t)}\,\rvtheta_0,\, \sigma^2(t)\mathbf{I})$
as in~\eqref{eq:vp-diff}. The conditional score network $\mathbf{s}_\phi(\rvtheta_t, \rvx, t)$ is trained via:
\begin{equation}
\LSBI = \E_{t,\, p(\rvtheta_0, \rvx),\, p(\rvtheta_t \mid \rvtheta_0)} \left[ \lambda(t) \| \mathbf{s}_\phi(\rvtheta_t, \rvx, t) - \nabla_{\rvtheta_t} \log p(\rvtheta_t \mid  \rvtheta_0) \|^2 \right],
\end{equation}
where $(\rvtheta_0, \rvx)$ is sampled by drawing $\rvtheta_0 \sim p(\rvtheta)$ from the prior and $\rvx \sim p(\rvx \mid \rvtheta_0)$ from the simulator. Once trained, $\mathbf{s}_\phi$ serves as an amortized estimator of the posterior score: for any observation $\rvx$, posterior samples $\rvtheta \sim p(\rvtheta\mid \rvx)$ can be generated by integrating a reverse-time SDE from Gaussian noise.

\subsection{Density Ratio Estimation}
\label{sec:dre}
Density Ratio Estimation (DRE) \citep{hermans2020likelihood, durkan2020contrastive} estimates density ratios using only samples from two distributions $p$ and $q$, without access to their density functions. By optimizing an objective, DRE recovers the ratio $r=\frac{q}{p}$ at optimality. Representative approaches include KLIEP \citep{sugiyama2007direct}, NRE \citep{hermans2020likelihood}, SDRE \citep{liu2019fisher}, and NRE-$\infty$ \citep{choi2022density}.

More generally, many DRE methods can be unified under generalized frameworks closely related to $f$-divergence estimation \citep{nguyen2010estimating, nowozin2016f}. Let $f$ be a convex, lower semi-continuous generator with convex conjugate $g = f^\star$, and let $h: \gX \to \operatorname{dom}(g)$ be a
variational function, parameterized by a neural network in practice. The objective


\begin{equation}
\gL(h)=-\E_{q}[h]+\E_{p}[g(h)],
\end{equation}
is minimized by a variational function $h^\star$ satisfying, under mild regularity conditions,
\begin{equation}
g'(h^\star)=\frac{q}{p},
\end{equation}

\subsection{Prior Adaptation in Amortized SBI}

A persistent challenge in amortized SBI is prior dependency, where the learned posterior is tied to the fixed prior $p(\rvtheta)$ used during training. To enable prior flexibility without retraining the simulator, two main research directions have emerged.

The first is meta-amortization, which trains models over a meta-distribution of priors to generalize across various target priors (see Appendix~\ref{app:meta_amortization} for details). The second, and more relevant to this work, is inference-time guidance, which adapts a pre-trained posterior during the sampling process. A representative framework is PriorGuide \citep{yang_2025_priorguide}, which serves as our primary baseline. PriorGuide incorporates a guidance term into a  diffusion-based SBI model using the density ratio $r(\rvtheta) = \frac{q(\rvtheta)}{p(\rvtheta)}$ between the target and training priors. The resulting guided score function is:
\begin{equation}
\label{eq:priorguide}
\nabla_{\rvtheta_t} \log q(\rvtheta_t \mid \rvx) 
=
\underbrace{\mathbf{s}_\phi(\rvtheta_t, \rvx, t)}_{\text{original score}}
+
\underbrace{
\nabla_{\rvtheta_t}
\log
\E_{p(\rvtheta_0 \mid \rvtheta_t, \rvx)}
\left[
r(\rvtheta_0)
\right]
}_{\text{guidance term}},
\end{equation}

where $\rvtheta_0$ and $\rvtheta_t$ denote the clean parameters and their noisy counterparts at diffusion step $t$, respectively.

Despite its efficiency, PriorGuide relies on several restrictive structural simplifications. Specifically, it approximates the ratio $r(\rvtheta_0)$ using a Gaussian Mixture Model (GMM) and assumes a Gaussian form for the reverse transition kernel $p(\rvtheta_0 \mid \rvtheta_t, \rvx)$. These heuristics can introduce systematic bias and lead to instability, particularly when the target prior $q(\rvtheta)$ has low overlap with $p(\rvtheta)$.

Our method instead employs DRE to approximate the guidance term directly and unbiasedly, bypassing these structural heuristics and enabling more robust posterior adaptation even under substantial prior mismatch.

\section{Methodology}
\label{sec:Methodology}

Our objective is to develop an unbiased and flexible guidance framework for inference-time prior adaptation. To achieve this, we seek an unbiased, simulation-efficient, and flexible estimator of the guidance term. Once such an estimator is obtained, a diffusion-based SBI model trained under prior $p(\rvtheta)$ can be rapidly adapted to a new target prior $q(\rvtheta)$ at inference time, enabling efficient generation of posterior samples from $q(\rvtheta \mid \rvx) \propto q(\rvtheta)p(\rvx \mid \rvtheta)$ as well as posterior predictive samples from $q(\rvx' \mid \rvx)$.

\subsection{Problem Formulation}
We consider the setting where a score-based diffusion model has been trained on a joint distribution $p(\rvtheta, \rvx) = p(\rvtheta)p(\rvx\mid\rvtheta)$, where $p(\rvtheta)$ denotes the training prior and $p(\rvx\mid\rvtheta)$ represents the simulator-defined likelihood. In the standard amortized paradigm, this pre-trained model serves as a foundation for rapid inference for posterior $p(\rvtheta \mid \rvx)$ or posterior predictive $p(\rvx' \mid \rvx)$. 

However, at inference time, a new target prior $q(\rvtheta)$ may be introduced to reflect updated scientific constraints or external domain knowledge. Assuming the target prior is defined as $q(\rvtheta) = r(\rvtheta)p(\rvtheta)$, where $r(\rvtheta)$ is a known density ratio, the corresponding target posterior becomes
\begin{equation}
q(\rvtheta \mid \rvx)
=
\frac{q(\rvtheta)p(\rvx\mid\rvtheta)}{q(\rvx)}
\propto
q(\rvtheta)p(\rvx\mid\rvtheta).
\end{equation}
Our primary objective is to perform inference on this posterior $q(\rvtheta\mid\rvx)$ and the corresponding posterior predictive distribution $q(\rvx'\mid\rvx)$.

Under the prior shift $q(\rvtheta) = r(\rvtheta)p(\rvtheta)$, the target joint distribution can be rewritten as
\begin{equation}
q(\rvtheta, \rvx)
=
r(\rvtheta)p(\rvtheta, \rvx),
\label{eq:joint_reweighting}
\end{equation}
since the likelihood $p(\rvx\mid\rvtheta)$ remains unchanged. This reweighting relationship forms the basis for the diffusion-space importance weighting strategy developed in Section~\ref{sec:guidance-term}.

\subsection{Guidance Derivation}
\label{sec:guidance-term}

For target posterior sampling, we aim to compute the guided score function $\nabla_{\rvtheta_t} \log q(\rvtheta_t \mid  \rvx)$, which can be decomposed as:
\begin{equation}
    \nabla_{\rvtheta_t} \log q(\rvtheta_t\mid \rvx)
= \nabla_{\rvtheta_t} \log p(\rvtheta_t\mid \rvx) + \nabla_{\rvtheta_t} \log \frac{q(\rvtheta_t\mid \rvx)}{p(\rvtheta_t\mid \rvx)} 
= \underbrace{\nabla_{\rvtheta_t} \log p(\rvtheta_t\mid \rvx)}_{\text{original score}} + \underbrace{\nabla_{\rvtheta_t} \log \frac{q(\rvtheta_t,\rvx)}{p(\rvtheta_t,\rvx)}}_{\text{guidance term}}. \label{eq:guidance-term}
\end{equation}

Notably, our guidance term is equivalent to the expression in Eq.~\ref{eq:priorguide} (see Appendix~\ref{app:guidance-derivation} for a formal proof). This reformulated joint log ratio
$\log\frac{q(\rvtheta_t,\rvx)}{p(\rvtheta_t,\rvx)}$
enables guidance estimation through DRE.

\subsection{DRE-based Guidance for Prior Change}

The reformulation in Eq.~\ref{eq:guidance-term} suggests that prior adaptation reduces to estimating the joint log ratio $\log \frac{q(\rvtheta_t,\rvx)}{p(\rvtheta_t,\rvx)}$. However, DRE objectives for this ratio typically require samples from both $q(\rvtheta_t, \rvx)$ and $p(\rvtheta_t, \rvx)$. Generating samples from the former involves a sequential process: first sampling $\rvtheta \sim q(\rvtheta)$, then invoking the simulator to obtain $\rvx \sim p(\rvx \mid \rvtheta)$, and finally applying the diffusion forward process to reach $\rvtheta_t \sim q(\rvtheta_t \mid \rvtheta)$. This process is computationally expensive and contradicts our primary objective of maintaining a fast and flexible inference pipeline. To avoid direct sampling from $q(\rvtheta_t, \rvx)$, we derive the following importance weighting identity, which we will use below to estimate the guidance term via two DRE methods: NRE and KLIEP.

\begin{proposition}
\label{proposition:diffusion_importance_weighting}

Define the prior ratio
$
r(\rvtheta_0)
=
\frac{q(\rvtheta_0)}
     {p(\rvtheta_0)}.
$
Further assume that the same forward diffusion process
$p(\rvtheta_t \mid \rvtheta_0)$
is applied under both priors.
Then, for any measurable and integrable function
$f(\rvtheta_t, \rvx, t)$,

\begin{equation}
\E_{q(\rvtheta_t,\rvx)}
[f(\rvtheta_t, \rvx, t)]
=
\E_{p(\rvtheta_t,\rvtheta_0,\rvx)}
\left[
r(\rvtheta_0)
f(\rvtheta_t, \rvx, t)
\right].
\label{eq:general_importance_weighting}
\end{equation}

\end{proposition}

\textbf{Sketch of proof:} Under our assumptions, we have the following identities
\begin{equation}
p(\rvtheta_t, \rvtheta_0, \rvx)
=
p(\rvtheta_t\mid\rvtheta_0)p(\rvtheta_0, \rvx),
\qquad
q(\rvtheta_t, \rvtheta_0, \rvx)
=
p(\rvtheta_t\mid\rvtheta_0)q(\rvtheta_0, \rvx).
\end{equation}
Thus, $$
q(\rvtheta_t, \rvtheta_0, \rvx) = 
p(\rvtheta_t\mid\rvtheta_0)q(\rvtheta_0, \rvx) = p(\rvtheta_t\mid\rvtheta_0)\underbrace{r(\rvtheta_0)p(\rvtheta_0, \rvx)}_{\text{Eq.} \ref{eq:joint_reweighting}} = r(\rvtheta_0) p(\rvtheta_t, \rvtheta_0, \rvx),$$
and the conclusion follows by substituting $q(\rvtheta_t, \rvx)$ in the expectation expression.

This reformulation enables simulation-free training by leveraging samples from the original joint distribution $p(\rvtheta_t, \rvtheta_0, \rvx) = p(\rvtheta_0, \rvx) p(\rvtheta_t \mid \rvtheta_0)$. Specifically, the required samples $(\rvtheta_0, \rvx)$ can be directly reused from the base SBI model's training set. By simply perturbing $\rvtheta_0$ via the diffusion forward kernel $p(\rvtheta_t \mid \rvtheta_0)$, the adaptation process circumvents any additional simulator calls.

We now instantiate this identity for NRE and KLIEP in turn.

\paragraph{NRE.}
NRE estimates the log density ratio 
$\log \frac{q(\rvtheta_t,\rvx)}{p(\rvtheta_t,\rvx)}$
by training a discriminator $a_\psi(\rvtheta_t, \rvx, t)$ with objective
\begin{equation}
\LNRE
= \E_t\!\left[
- \E_{q(\rvtheta_t,\rvx)}\!\left[\log S(a_\psi)\right]
- \E_{p(\rvtheta_t,\rvx)}\!\left[\log\bigl(1 - S(a_\psi)\bigr)\right]
\right],
\label{eq:NRE loss}
\end{equation}
where $a_\psi$ is shorthand for $a_\psi(\rvtheta_t, \rvx, t)$ and $S$ is sigmoid function. In practice, the loss is estimated using an equal number of samples drawn from each of $q(\rvtheta_t, \rvx)$ and $p(\rvtheta_t, \rvx)$.

Applying Prop.~\ref{proposition:diffusion_importance_weighting} to Eq.~\ref{eq:NRE loss}, the objective becomes:

\begin{equation}
\LNRE
=
\E_t\left[-\E_{p(\rvtheta_t,\rvtheta_0,\rvx)}
[r(\rvtheta_0)\log S(a_\psi)]
-
\E_{p(\rvtheta_t,\rvx)}
[\log(1-S(a_\psi))]\right].
\label{eq:nre-a-loss}
\end{equation}

Finally, the guidance term is recovered via the gradient of the optimum $a^\star$: $\nabla_{\rvtheta_t} \log \frac{q(\rvtheta_t,\rvx)}{p(\rvtheta_t,\rvx)} = \nabla_{\rvtheta_t} a^\star(\rvtheta_t, \rvx, t)$.

\paragraph{KLIEP.}
KLIEP estimates the same log density ratio by training an estimator $a_\psi(\rvtheta_t, \rvx, t)$ via the following objective
\begin{equation}
\LKLIEP
=
\E_t\left[-\E_{q(\rvtheta_t,\rvx)}[a_\psi]
+
\log
\E_{p(\rvtheta_t,\rvx)}[\exp a_\psi]\right].
\end{equation}

Applying Prop.~\ref{proposition:diffusion_importance_weighting} yields
\begin{equation}
\LKLIEP
=
\E_t\left[-\E_{p(\rvtheta_t,\rvtheta_0,\rvx)}
[r(\rvtheta_0)a_\psi]+\log
\E_{p(\rvtheta_t,\rvx)}[\exp a_\psi]\right].
\label{eq:kliep-loss}
\end{equation}

which, like $\LNRE$ in Eq.~\ref{eq:nre-a-loss}, admits the same simulation-free training procedure. The guidance term is likewise recovered via the gradient of the optimum $a^\star$: $\nabla_{\rvtheta_t} \log \frac{q(\rvtheta_t,\rvx)}{p(\rvtheta_t,\rvx)} = \nabla_{\rvtheta_t} a^\star(\rvtheta_t, \rvx, t)$.

The two examples above illustrate an important observation: although derived from different objectives, both NRE and KLIEP ultimately estimate the same diffusion-space log density ratio $\log \frac{q(\rvtheta_t,\rvx)}{p(\rvtheta_t,\rvx)}.$

More broadly, many density ratio estimation methods can be interpreted within a unified variational DRE framework, where different objectives recover the same ratio at optimality. This observation motivates a general estimator-agnostic guidance framework for inference-time prior adaptation.

\subsection{General DRE-based Guidance Framework}

We now extend the general DRE objectives of Section~\ref{sec:dre} to the diffusion joint ratio, yielding a general framework for guidance term estimation.

\begin{theorem}
\label{thm:general_ratio}

Define the prior ratio $r(\rvtheta_0) = \frac{q(\rvtheta_0)}{p(\rvtheta_0)}$. Consider the generalized density ratio estimation objective:
\begin{equation}
h^\star = \arg\min_h \left\{ \E_t\left[- \E_{p(\rvtheta_t,\rvtheta_0,\rvx)} \left[ r(\rvtheta_0) h(\rvtheta_t, \rvx, t) \right] + \E_{p(\rvtheta_t,\rvx)} \left[ g(h(\rvtheta_t, \rvx, t)) \right]\right] \right\},
\label{eq:general_loss}
\end{equation}
where $g:\R\to\R$ is differentiable and strictly convex.

Then the optimal solution satisfies:
\begin{equation}
g'\!\left(h^\star(\rvtheta_t, \rvx, t)\right) =  \frac{q(\rvtheta_t,\rvx)}{p(\rvtheta_t,\rvx)},
\end{equation}
Consequently, the corresponding guidance term is recovered via:
\begin{equation}
\nabla_{\rvtheta_t} \log \left( g'\!\left(h^\star(\rvtheta_t, \rvx, t)\right) \right) = \nabla_{\rvtheta_t} \log \frac{q(\rvtheta_t,\rvx)}{p(\rvtheta_t,\rvx)}.
\end{equation}

\end{theorem}
This general form is instantiated via $h(\rvtheta_t, \rvx, t) = \sigma_g(a_\psi(\rvtheta_t, \rvx, t))$, where $a_\psi\in\R$ is an unconstrained network output and $\sigma_g:\R\to\mathrm{dom}(g)$ is an activation function determined by the choice of generator $g$. NRE corresponds to $(\sigma_g, g) = (\log S,\, -\log(1-e^h))$, recovering Eq.~\ref{eq:nre-a-loss} exactly as a special case. 
KLIEP shares the same underlying network $a_\psi$ but its global log-partition structure falls outside this pointwise template; we show separately that it recovers the same guidance term at optimality (Appendix~\ref{app:kliep_special_case}). More generally, different choices of $g$ induce different density ratio estimators within this shared framework, which we summarize in Table~\ref{tab:f_divergence} (Appendix~\ref{app:common_instances}). We provide the formal proof of Theorem~\ref{thm:general_ratio} in Appendix~\ref{app:general_ratio}.

Beyond its generality, this framework inherits a natural stabilization mechanism from diffusion. By performing DRE over diffused states $\rvtheta_t$ rather than clean samples, the inherent Gaussian noise ensures that the distributions remain well-supported. This prevents the density ratio from becoming ill-defined, ensuring robust guidance even when the target and training priors have mismatched supports.

\section{Experiments}
\label{sec:Experiments}

We empirically evaluate DRE-based guidance on SBI problems requiring test-time prior adaptation, for both posterior and posterior predictive inference, using Simformer \citep{gloeckler2024all} as the base diffusion model.
In Section~\ref{sec:numerical_experiments}, we assess performance on several numerical SBI benchmarks, comparing our approach with existing methods for prior adaptation. 
In Section~\ref{sec:lightcurve}, we demonstrate the applicability of our method to a real-world light curve task via sequential posterior updating, using a Bayesian update algorithm built upon our framework (see Algorithm~\ref{alg:prior_guided} in Appendix~\ref{app:alg}).

\subsection{Numerical Experiments}
\label{sec:numerical_experiments}

We evaluate our method on a diverse set of SBI benchmarks, focusing on its ability to adapt to prior change at test time. We consider two regimes:
\begin{itemize}
    \item \textbf{In-distribution (ID):} The training prior provides sufficient support for the target prior.
    \item \textbf{Out-of-distribution (OOD):} The target prior shift results in poor overlap with the training prior.
\end{itemize}

We compare our DRE-based guidance, implemented via NRE and KLIEP, against the vanilla Simformer (no adaptation), PriorGuide, and weighted score matching (WSM), a baseline that learns an additive score correction via a prior-ratio-weighted denoising score matching objective (Appendix~\ref{app:wsm}). Our method, PriorGuide, and WSM use the same setting and are therefore directly comparable: all three involve only lightweight test-time training, leave the base model unchanged, require no new simulator calls, and access the prior change only through the ratio
$q(\rvtheta_0)/p(\rvtheta_0)$. We additionally report two reference methods that are not directly comparable, Simformer+SIR and analytic guidance (Appendix~\ref{app:ref-method}).

\subsubsection{Posterior Inference under Prior Shift}
\label{sec:posterior}

We evaluate posterior inference performance on five representative SBI benchmarks: Two Moons \citep{lueckmann2021benchmarking}, SLCP \citep{papamakarios2019sequential}, Gaussian Linear 6D \citep{lueckmann2021benchmarking}, OUP \citep{uhlenbeck1930theory}, and the Turin model \citep{turin1972statistical} (details in Appendix~\ref{app:simulators}). These tasks span low- to high-dimensional spaces and include both static and time-series data. We consider both ID and OOD regimes.

For each regime, we randomly draw 10 independent observations from the target marginal distribution of the observations. For each observation, we generate 5,000 posterior samples and repeat this sampling process across 3 independent runs to ensure statistical reliability. Performance is evaluated using C2ST, MMD, $W_2$, and RMSE. Due to space constraints, we present C2ST scores in Figure~\ref{fig:c2st}, while full results are reported in Table~\ref{tab:posterior} of Appendix~\ref{app:experiment_res}.

Overall, our methods (NRE and KLIEP) achieve the best or are not significantly worse than the best-performing method on most tasks, with the largest margins in the OOD regime: on Gaussian Linear 6D and OUP, our DRE-based guidance attains C2ST scores of $\sim$0.58 and $\sim$0.54, versus $\sim$0.69 and $\sim$0.65 for PriorGuide, and on SLCP (OOD), our NRE is the only method in the best-performing group. By comparison, WSM generally underperforms our methods and shows high variance on some tasks, suggesting that learning the score correction by weighted regression is harder than estimating the density ratio directly.

Figure~\ref{fig:c2st} shows two failure cases, Two Moons (OOD) and Turin, which arise for distinct reasons; we analyse both in Appendix~\ref{app:failure}.

\begin{figure}
    \centering
    \includegraphics[width=1.0\linewidth]{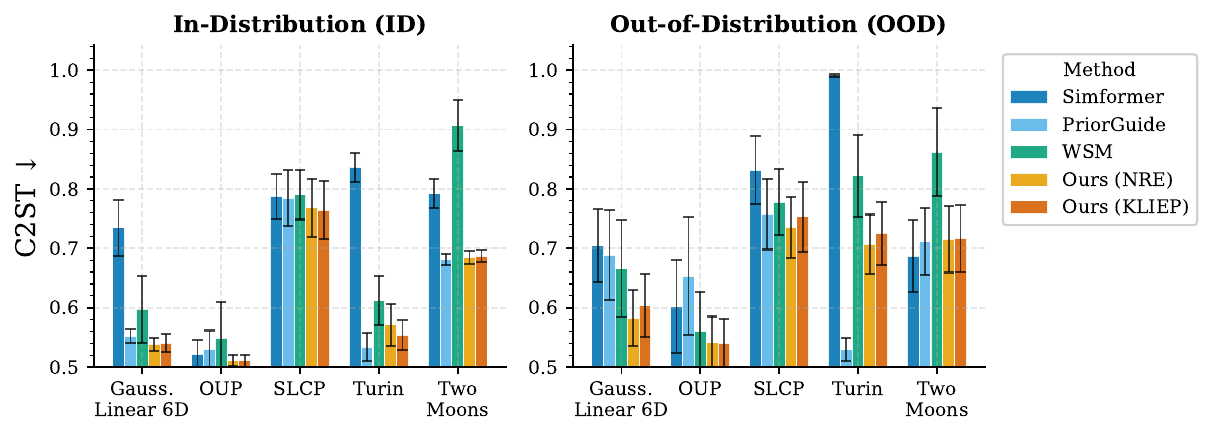}
\caption{
C2ST scores ($\downarrow$ better, chance level = 0.5) across five benchmark tasks under ID and OOD settings. 
Our methods (NRE and KLIEP) are best or not significantly worse than the best method on most tasks, with the largest gains under OOD shift; Two Moons and Turin are exceptions, analysed in Appendix~\ref{app:failure}.
}
    \label{fig:c2st}
\end{figure}

\subsubsection{Posterior Predictive under Prior Shift}
\label{sec:predictive}

We evaluate posterior predictive performance under both ID and OOD regimes on two time-series SBI tasks: the OUP and Turin models. For all tasks, the first one-third of trajectories are provided as observed, and the remaining two-thirds are predicted.

While RMSE is comparable for most methods and settings (Table~\ref{tab:posterior_predictive_results} in Appendix~\ref{app:experiment_res}), qualitative evaluation in Figure~\ref{fig:predictive} reveals distinct failure modes. In OUP (OOD) (Fig.~\ref{fig:predictive}a), PriorGuide's predictive distribution is systematically biased, a failure that RMSE also registers. In Turin (ID) (Fig.~\ref{fig:predictive}b), all methods attain comparable RMSE, yet PriorGuide misplaces predictive mass in the upper tail. 
These observations are consistent with the posterior results: PriorGuide's structural approximation introduces bias, while our method and WSM show neither failure mode even where RMSE is comparable.
Additional results are provided in Appendix~\ref{app:experiment_res}.

\begin{figure}[h]
    \centering
    \includegraphics[width=1.0\linewidth]{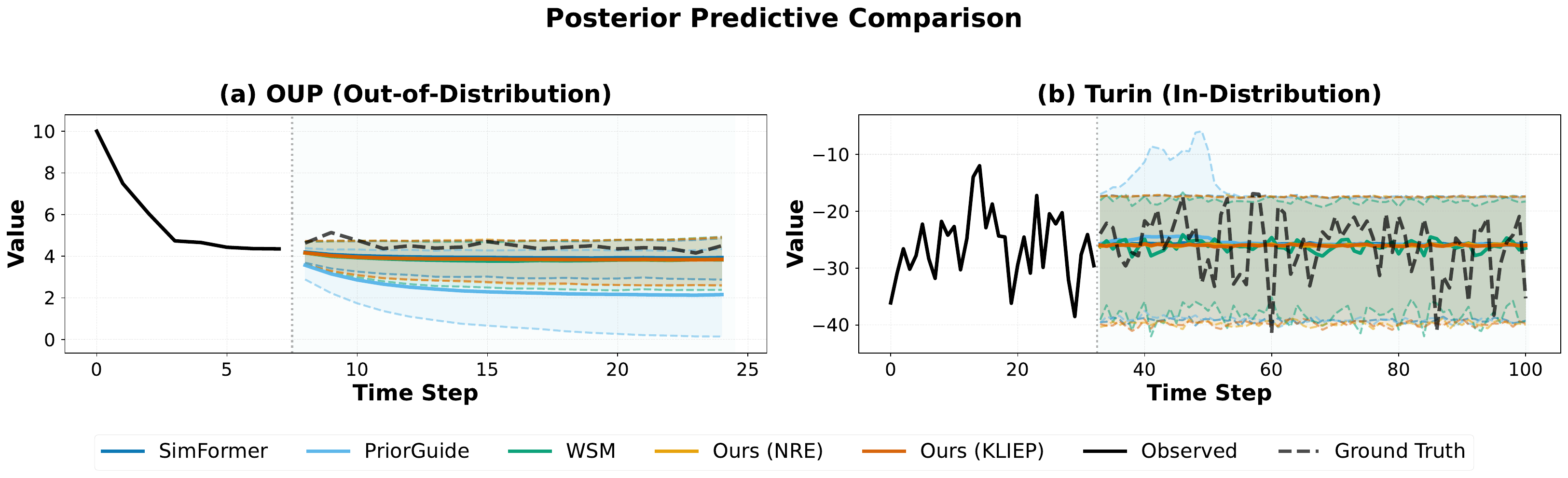}
    \caption{Posterior predictive comparison on time series forecasting. (a) OUP (OOD), (b) Turin (ID). Solid lines show predictive means, dashed lines and shaded regions indicate $95\%$ confidence intervals. Black lines represent observed data (solid) and ground truth (dashed). Vertical line marks the prediction boundary.}
    \label{fig:predictive}
\end{figure}

\subsection{Planet Light Curve Posterior Update}
\label{sec:lightcurve}

We evaluate our method on a sequential inference task using stellar light curves from the TESS mission \citep{ricker2015transiting}. Specifically, we consider three exoplanet systems: GJ 436 b, WASP-39 b, and WASP-126 b. For each system, we extract a sequence of photometric observations $\{\rvx^{(n)}_{\text{obs}}\}_{n=1}^{N}$, where each $\rvx^{(n)}_{\text{obs}} \in \R^{128}$ is a normalized light-curve segment centered on a planetary transit event (see Figure~\ref{fig:lc_obs} in Appendix~\ref{app:lc_obs} for an
example sequence of observation windows for WASP-39\,b). The goal is to infer 8 parameters $\rvtheta \in \R^8$, including orbital geometry $(k, a, \mathrm{inc}, P)$, stellar effects $(u_1, u_2)$, and observational parameters $(t_0, \sigma)$ (detailed descriptions in Table~\ref{tab:lc_priors}). We report results on $(k, a, \mathrm{inc}, P)$.

\paragraph{Sequential Posterior Guidance.} 
As new observations become available, the posterior from previous steps serves as the prior for subsequent inference. Assuming conditional independence, the posterior evolves as
\[
p(\rvtheta \mid \rvx^{(1:n+1)}) \propto p(\rvx^{(n+1)} \mid \rvtheta)\, p(\rvtheta \mid \rvx^{(1:n)}).
\]
We train an amortized diffusion model for $p(\rvtheta \mid \rvx)$ and use DRE to estimate $r^{(n)}(\rvtheta) \approx p(\rvtheta \mid \rvx^{(1:n)})/p(\rvtheta)$, enabling posterior guidance at each step. This yields a sequence $p(\rvtheta), p(\rvtheta \mid \rvx^{(1)}), \dots, p(\rvtheta \mid \rvx^{(1:N)})$. The full sequential update procedure is given in Algorithm~\ref{alg:prior_guided} in Appendix~\ref{app:alg}.

\paragraph{Results.}
Figure~\ref{fig:lc_1} shows that our method shifts and concentrates the posterior toward the reference values as observations accumulate, demonstrating stable sequential updating. In contrast, PriorGuide becomes unstable after a few updates, producing over-dispersed or collapsed posteriors. 
This trend is consistent across other systems (Figures~\ref{fig:lc_2} and~\ref{fig:lc_3} in Appendix~\ref{app:lc_res}) and is further reflected in the RMSE comparisons (Figure~\ref{fig:lc_rmse} in Appendix~\ref{app:lc_res}): our method maintains low and stable error, with slight improvements as more observations are incorporated, while PriorGuide exhibits progressively increasing error across sequential updates, suggesting the accumulation of guidance bias over time.

\begin{figure}[h]
    \centering
    \includegraphics[width=1\linewidth]{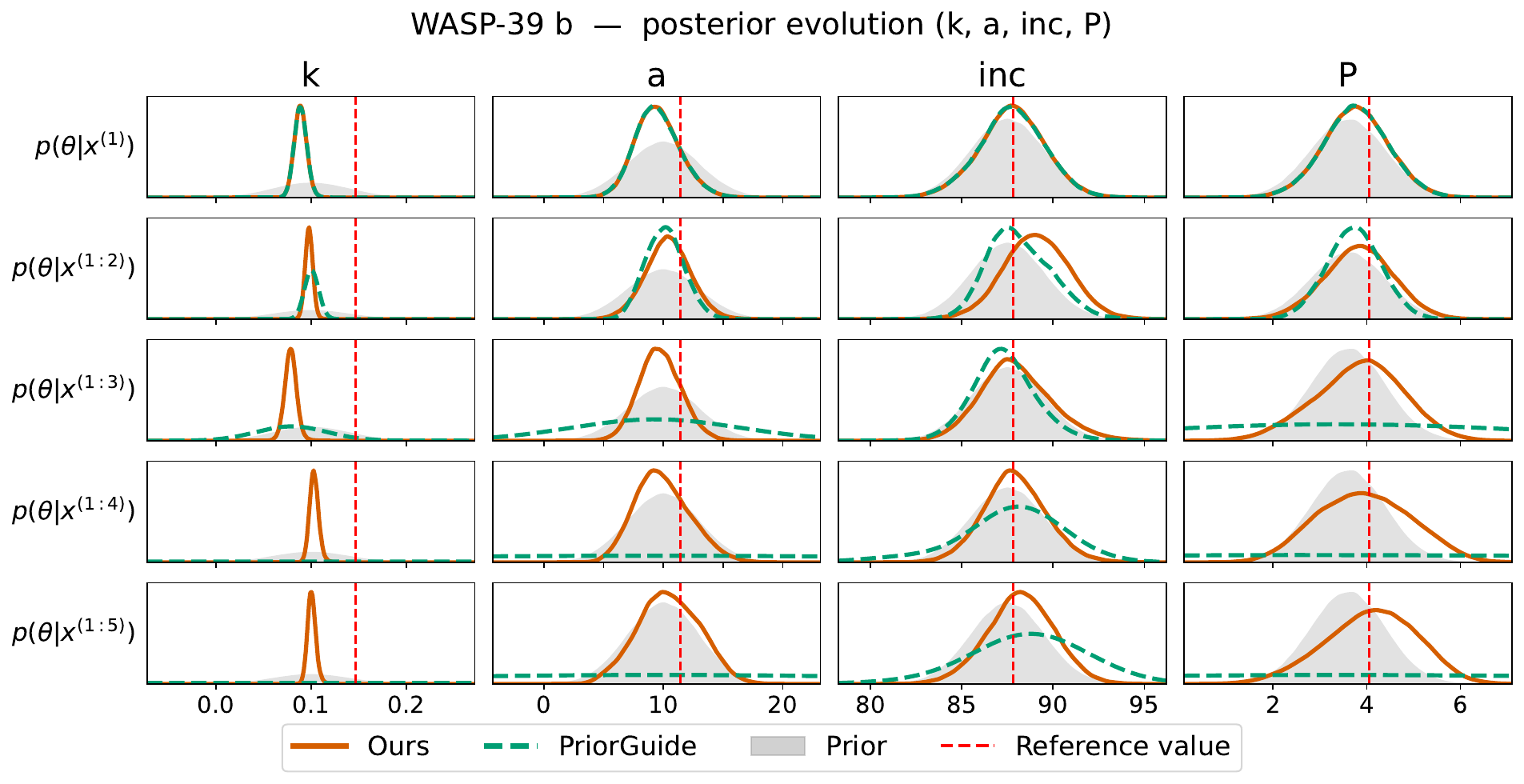}
    \caption{Sequential posterior evolution for WASP-39b over $5$ observations. Row $n$ shows $p(\rvtheta \mid \rvx^{(1:n)})$ for $k$, $a$, inc, and $P$. Grey: prior; red dashed: reference value from \citep{faedi2011wasp}.}
    \label{fig:lc_1}
\end{figure}

\section{Discussion}
\label{sec:Discussion}

Our results highlight that accurate guidance estimation is key to prior adaptation in amortized SBI. By learning the gradient of the joint density ratio via DRE, our method avoids the structural approximations used in PriorGuide and provides an unbiased estimate of the guidance term. Empirically, we observe gains on the majority of tasks, particularly under out-of-distribution (OOD) prior shifts, where PriorGuide degrades due to biased guidance.

Beyond accuracy, the two methods also differ in cost. Both avoid additional simulator calls but fit a lightweight model per target prior, a ratio estimator in our case and a GMM in theirs, so neither is purely test-time. At inference, however, our sampling requires none of the Langevin correction PriorGuide relies on, a cost recurring at every observation.

Performance of our method depends on the quality of the learned ratio estimator, which degrades in high dimensions; compressing observations with a learned embedding network is the standard remedy and an important direction for future work. We also observe a case where the prior ratio is nearly flat over the posterior support, in which guidance carries little signal for any method.

Overall, our framework offers a practical test-time-training route to prior adaptation without retraining the base model, particularly useful in scientific workflows with evolving prior knowledge.

\section*{Acknowledgements}

We thank Sotiria Fotopoulou and Grant Stevens for helpful suggestions on the planetary light-curve application. 

\bibliographystyle{plain}
\bibliography{sample}

\appendix
\newpage

\section{Meta-Amortization Methods for Prior Change}
\label{app:meta_amortization}
Meta-amortization strategies aim to learn a functional mapping that accepts both observed data $\rvx$ and a formal prior specification $p(\rvtheta)$ as inputs. This is achieved by training the model over a ``meta-prior''---a distribution or a predefined set of possible prior specifications. For instance, \textbf{Sensitivity-aware SBI} \citep{elsemuller2023sensitivity} trains across a discrete set of alternative priors to provide efficient sensitivity analysis. More general frameworks like the \textbf{Distribution Transformer (DT)} \citep{whittle_2025_distribution} and the \textbf{Amortized Conditioning Engine (ACE)} \citep{chang_2024_amortized} employ attention-based architectures to process prior specifications represented as GMMs or histograms. While these methods enable rapid adaptation at runtime, their flexibility is fundamentally bounded by the meta-prior seen during training, often struggling to generalize to prior geometries that deviate significantly from the pre-training set.

\section{Equivalence of Guidance Terms in Eq.\ref{eq:priorguide}}
\label{app:guidance-derivation}

In this section, we show the equivalence between our guidance term in Eq.~\ref{eq:guidance-term} and the one in Eq.~\ref{eq:priorguide}. By Eq.~\ref{eq:joint_reweighting}, the target and original joint distributions are related by $q(\rvtheta_0, \rvx) = r(\rvtheta_0) p(\rvtheta_0, \rvx)$. Since the transition kernel $p(\rvtheta_t \mid \rvtheta_0, \rvx)$ is independent of the prior, the marginal joint distribution $q(\rvtheta_t, \rvx)$ can be simplified as:
\begin{equation}
q(\rvtheta_t, \rvx) = \int r(\rvtheta_0) p(\rvtheta_0, \rvx) p(\rvtheta_t \mid \rvtheta_0, \rvx) d\rvtheta_0 = \int r(\rvtheta_0) p(\rvtheta_t, \rvtheta_0, \rvx) d\rvtheta_0.
\end{equation}
The ratio between the target and original joint distributions is then:
\begin{equation}
\begin{aligned}
\frac{q(\rvtheta_t,\rvx)}{p(\rvtheta_t,\rvx)} &= \frac{\int r(\rvtheta_0) p(\rvtheta_t, \rvtheta_0, \rvx) d\rvtheta_0}{p(\rvtheta_t, \rvx)} \\
&= \int r(\rvtheta_0) \frac{p(\rvtheta_t, \rvtheta_0, \rvx)}{p(\rvtheta_t, \rvx)} d\rvtheta_0 \\
&= \int r(\rvtheta_0) p(\rvtheta_0 \mid \rvtheta_t, \rvx) d\rvtheta_0 \\
&= \E_{p(\rvtheta_0 \mid \rvtheta_t, \rvx)} \left[ r(\rvtheta_0) \right].
\end{aligned}
\end{equation}
Taking the gradient of the logarithm on both sides yields the equivalence:
\begin{equation}
\nabla_{\rvtheta_t} \log \frac{q(\rvtheta_t,\rvx)}{p(\rvtheta_t,\rvx)} = \nabla_{\rvtheta_t} \log \E_{p(\rvtheta_0 \mid \rvtheta_t, \rvx)} \left[ r(\rvtheta_0) \right].
\end{equation}

\section{Proof of Theorem~\ref{thm:general_ratio}}
\label{app:general_ratio}

\begin{proof}
Following the variational framework for density ratio estimation \citep{nguyen2010estimating, sugiyama2007direct}, we consider the problem of estimating the ratio of two distributions by minimizing a divergence-based objective. Since the objective in Eq.~\ref{eq:general_loss} is separable across diffusion timesteps $t$, we can analyze the functional pointwise for a fixed $t$:
\begin{equation}
\gL(h) = - \E_{q(\rvtheta_t,\rvx)} [h(\rvtheta_t, \rvx, t)] + \E_{p(\rvtheta_t,\rvx)} [g(h(\rvtheta_t, \rvx, t))].
\end{equation}
By expressing the first term as an expectation over the training distribution $p(\rvtheta_t, \rvx)$, the functional can be rewritten as:
\begin{equation}
\gL(h) = \E_{p(\rvtheta_t,\rvx)} \left[ -\frac{q(\rvtheta_t,\rvx)}{p(\rvtheta_t,\rvx)} h(\rvtheta_t, \rvx, t) + g(h(\rvtheta_t, \rvx, t)) \right].
\end{equation}
For a fixed state $(\rvtheta_t, \rvx, t)$, we minimize the integrand $\ell(h) = -\frac{q(\rvtheta_t,\rvx)}{p(\rvtheta_t,\rvx)} h + g(h)$. Given that $g$ is strictly convex, the unique minimizer $h^\star$ is obtained by setting the first-order derivative to zero:
\begin{equation}
\frac{\partial \ell(h)}{\partial h} = -\frac{q(\rvtheta_t,\rvx)}{p(\rvtheta_t,\rvx)} + g'(h) = 0 \implies g'(h^\star) = \frac{q(\rvtheta_t,\rvx)}{p(\rvtheta_t,\rvx)}. \label{eq:optimality-condition-final}
\end{equation}

To facilitate simulation-free training, we apply the change-of-measure identity from Proposition~\ref{proposition:diffusion_importance_weighting}, allowing the expectation over $q(\rvtheta_t, \rvx)$ to be computed using samples from the original joint distribution $p(\rvtheta_t, \rvtheta_0, \rvx)$:
\begin{equation}
\E_{q(\rvtheta_t,\rvx)} [h(\rvtheta_t, \rvx, t)] = \E_{p(\rvtheta_t, \rvtheta_0, \rvx)} [r(\rvtheta_0) h(\rvtheta_t, \rvx, t)],
\end{equation}
where $r(\rvtheta_0) = q(\rvtheta_0)/p(\rvtheta_0)$ is the prior ratio. Substituting this into the variational functional yields the final objective:
\begin{equation}
h^\star = \arg\min_h \E_t \left[ - \E_{p(\rvtheta_t, \rvtheta_0, \rvx)} [r(\rvtheta_0) h(\rvtheta_t, \rvx, t)] + \E_{p(\rvtheta_t,\rvx)} [g(h(\rvtheta_t, \rvx, t))] \right].
\end{equation}
Taking the gradient of the log-ratio recovers the required guidance term:
\begin{equation}
\nabla_{\rvtheta_t} \log \frac{q(\rvtheta_t,\rvx)}{p(\rvtheta_t,\rvx)} = \nabla_{\rvtheta_t} \log \left( g'\!\left(h^\star(\rvtheta_t, \rvx, t)\right) \right).
\end{equation}
This completes the proof.
\end{proof}

\section{Special cases of General DRE-based Guidance}
\label{app:dre_special_case}

\subsection{Derivation of KLIEP as an Extended Instance}
\label{app:kliep_special_case}
Unlike NRE, the KLIEP loss $\LKLIEP$ in Eq.~\ref{eq:kliep-loss} is not a pointwise instance of Theorem~\ref{thm:general_ratio}. We show below that it can be reduced to the same variational form.

Fix $t$ and write $a = a_\psi(\cdot,\cdot,t)$. Let
\begin{equation}
Z = \E_{p(\rvtheta_t,\rvx)}[\exp a],
\qquad
b = a - \log Z + 1.
\end{equation}
Since $Z$ is constant with respect to $(\rvtheta_t,\rvx)$,
\begin{equation}
e^{b-1} = \frac{e^{a}}{Z},
\qquad
\E_{p(\rvtheta_t,\rvx)}[\exp(b-1)] = \frac{\E_{p(\rvtheta_t,\rvx)}[\exp a]}{Z} = 1.
\label{eq:kliep-h-constraint}
\end{equation}
Moreover,
\begin{equation}
\E_{q(\rvtheta_t,\rvx)}[b] = \E_{q(\rvtheta_t,\rvx)}[a] - \log Z + 1.
\end{equation}
Therefore, 
\begin{align}
\LKLIEP &= \E_t\Bigl[-\E_{q(\rvtheta_t,\rvx)}[a] + \log\E_{p(\rvtheta_t,\rvx)}[\exp a]\Bigr]
= \E_t\Bigl[-\E_{q(\rvtheta_t,\rvx)}[b] - \log Z + 1 + \log Z\Bigr] \notag\\
&= \E_t\Bigl[-\E_{q(\rvtheta_t,\rvx)}[b] + \E_{p(\rvtheta_t,\rvx)}[\exp(b-1)]\Bigr] = \E_t\Bigl[-\E_{q(\rvtheta_t,\rvx)}[b] + \E_{p(\rvtheta_t,\rvx)}[g(b)]\Bigr],
\end{align}
where $g(b) := \exp(b-1)$ and the last line uses Eq.~\ref{eq:kliep-h-constraint} which is exactly the variational form of Theorem~\ref{thm:general_ratio}.

The derivative of the generator is $g'(b) = \exp(b-1)$. By Theorem~\ref{thm:general_ratio}, the optimal solution satisfies $g'(b^\star) = q/p$, i.e.,
\begin{equation}
b^\star(\rvtheta_t,\rvx,t) = 1 + \log\frac{q(\rvtheta_t,\rvx)}{p(\rvtheta_t,\rvx)}.
\end{equation}
Using the definition of $h$, the optimal discriminator satisfies
$a^\star(\rvtheta_t,\rvx,t) = \log\frac{q(\rvtheta_t,\rvx)}{p(\rvtheta_t,\rvx)} + \log Z,$
where $\log Z$ is constant with respect to $(\rvtheta_t,\rvx)$ (it may depend on $t$). Consequently,
$
\nabla_{\rvtheta_t} a^\star(\rvtheta_t,\rvx,t) = \nabla_{\rvtheta_t}\log\frac{q(\rvtheta_t,\rvx)}{p(\rvtheta_t,\rvx)},
$
showing that KLIEP recovers the exact log-density ratio up to an additive constant, and therefore yields the correct guidance term.

\subsection{Common Instantiations of DRE Framework}
\label{app:common_instances}
\begin{table}[H]
\centering
\caption{Instantiations of the generalized DRE framework. Various choices of the convex generator $g(v)$ induce distinct variational objectives, yet they consistently recover the same optimal density ratio $q/p$.}
\label{tab:f_divergence}
\begin{tabular}{lcc}
\toprule
Method & $g(v)$ & Optimal $h^\star$ \\
\midrule
NRE & $-\log(1-e^v)$ &  $\log \frac{q}{p+q}$ \\
Pearson $\chi^2$ & $\frac{1}{4}v^2+v$ & $2(\frac{q}{p}-1)$ \\
Reverse KL & $-1-\log(-v)$  & $-p/q$ \\
Hellinger & $\frac{v}{1-v}$ & $1-\sqrt{p/q}$ \\
\bottomrule
\end{tabular}
\end{table}

\section{Baselines and Reference Methods}
\subsection{Weighted Score Matching (WSM)}
\label{app:wsm}
WSM keeps the base score $s_{\mathrm{base}}$ frozen and learns an additive
correction network $u$ by minimizing the prior-ratio-weighted denoising
score matching loss
\begin{equation}
\mathcal{L}(u)=\mathbb{E}_{t,\,p(\rvtheta_t,\rvtheta_0,\rvx)}\!\left[
\frac{q(\rvtheta_0)}{p(\rvtheta_0)}
\bigl\| s_{\mathrm{base}}(\rvtheta_t,\rvx,t)+u(\rvtheta_t,\rvx,t)
-\nabla_{\rvtheta_t}\log p(\rvtheta_t\mid\rvtheta_0)\bigr\|^2\right].
\label{eq:wsm_loss}
\end{equation}
At sampling time, the adapted score $s_{\mathrm{base}}+u$ is used to generate new posterior samples. Since the expectation is taken under the
original training distribution, the loss is estimated on the existing
training set with weights $r(\rvtheta_0)=q(\rvtheta_0)/p(\rvtheta_0)$, so WSM
requires no new simulator calls. The reweighting principle follows
SNPE-B \citep{lueckmann2017flexible}, which importance-weights the NPE objective
to correct for proposal--prior mismatch. When $s_{\mathrm{base}}$ equals the base score $\nabla_{\rvtheta_t}\log p(\rvtheta_t\mid \rvx)$, the optimal additive correction network $u^*$ coincides with the guidance term $\nabla_{\rvtheta_t} \log \frac{q(\rvtheta_t,\rvx)}{p(\rvtheta_t,\rvx)}$.

\subsection{Reference Methods}
\label{app:ref-method}

Two additional reference methods are excluded from the ranked comparison.

\textbf{Simformer+SIR} reweights a finite pool of samples drawn from the base model by the prior ratio and resamples from that pool, so every returned sample is an exact copy of one already produced under the base prior. When the target posterior is poorly covered by the pool, the importance weights concentrate on a few samples and the resampled set collapses onto a small number of distinct points, severely limiting posterior sample diversity — a practical drawback in its own right. This mechanism is categorically different from our method and PriorGuide, both of which produce fresh samples at test time, and it interacts with the evaluation metrics in an asymmetric way. Exact duplication leaves the kernel mean embedding and the transport cost almost unchanged, so MMD and $W_2$ remain small, while a classifier trained to separate the two sample sets exploits the repeated points directly, systematically inflating C2ST. As a result, no single ranking of Simformer+SIR against our method and the remaining baselines is consistent across metrics, and we therefore report it as an importance-resampling reference rather than as an entry in the ranked comparison.

\textbf{Analytic guidance}, available only for the Gaussian Linear task where the guidance term admits a closed form, serves as an oracle reference representing performance with exact guidance.

\section{Algorithm: Sequential Posterior Update via DRE-based Diffusion Guidance (NRE Instantiation)}
\label{app:alg}
\begin{algorithm}[H]
\caption{Sequential Posterior Update via DRE-based Diffusion Guidance}
\label{alg:prior_guided}
\begin{algorithmic}[1]

\vspace{0.5em}
\STATE \textbf{Offline:}

Simulate
$\gD=\{(\rvtheta_0^{(k)},\rvx^{(k)})\}_{k=1}^{K}
\sim p(\rvtheta_0)p(\rvx\mid\rvtheta_0)$.

Generate original prior samples
$\Theta^{(0)}=\{{\rvtheta}_0^{(0,j)}\}_{j=1}^{M}\sim p(\rvtheta_0)$.

Using $\gD$, train a score-based diffusion model with $T$ diffusion steps with score $\mathbf{s}_\phi(\rvtheta_t, \rvx, t)\approx
\nabla_{\rvtheta_t}\log p(\rvtheta_t\mid \rvx)$.

\vspace{0.5em}
\STATE \textbf{Online:} 

Sample $\Theta^{(1)} = \{\rvtheta_0^{(1,j)}\}_{j=1}^M \sim q^{(1)}(\rvtheta_0) \approx p(\rvtheta_0 \mid \rvx^{(1)}_{\text{obs}})$.

Given $\rvx^{(1)}_{\text{obs}},\rvx^{(2)}_{\text{obs}},\ldots,\rvx^{(N)}_{\text{obs}}$. 

\vspace{0.5em}
\FOR{$n=2,3,\ldots,N$}
    \vspace{0.5em}

\STATE \textit{// estimate prior ratio}
    \STATE Estimate prior ratio
    $\hat{r}^{(n-1)}(\rvtheta_0)\approx \frac{q^{(n-1)}(\rvtheta_0)}{p(\rvtheta_0)}$
    using previous posterior samples
    $\Theta^{(n-1)}$
    and original prior samples $\Theta^{(0)}$.
    \vspace{0.5em}
    \STATE \textit{// train estimator for guidance term}
    \STATE Train log-ratio estimator $a_\psi^{(n)}(\rvtheta_t, \rvx, t)$ using the NRE objective
    \[
   \gL^{(n)}_{\mathrm{NRE}}
    =
    \E_t\left[-
    \E_{p(\rvtheta_t,\rvtheta_0,\rvx)}
    \!\left[
    \hat r^{(n-1)}(\rvtheta_0)
    \log S(a_\psi^{(n)})
    \right]
    -
    \E_{p(\rvtheta_t,\rvx)}
    \!\left[
    \log(1-S(a_\psi^{(n)}))
    \right]\right].
    \]
    \STATE \textit{// guided sampling}
    \STATE Initialize 
    $\Theta_T^{(n)} = \{\rvtheta_T^{(n,j)}\}_{j=1}^M\sim\gN(\mathbf{0},\rmI)$.
    \FOR{$t=T,\ldots,1$}
        \STATE $\tilde{\mathbf{s}}^{(n)} \leftarrow \mathbf{s}_\phi(\Theta^{(n)}_t,\rvx^{(n)}_{\text{obs}},t) + \nabla_{\rvtheta_t} a_\psi^{(n)}(\Theta^{(n)}_t,\rvx^{(n)}_{\text{obs}}, t)$
        \STATE $\Theta^{(n)}_{t-1} \leftarrow \textsc{SdeSolverStep}(\Theta^{(n)}_t,\tilde{\mathbf{s}}^{(n)},\Delta t)$ \COMMENT{reverse-time SDE step with step size $\Delta t$}
    \ENDFOR
    \vspace{0.6em}
    \STATE $\Theta^{(n)} = \{\rvtheta_0^{(n,j)}\}_{j=1}^M \sim q^{(n)}(\rvtheta_0) \approx p(\rvtheta_0\mid \rvx^{(1:n)}_{\text{obs}})$
\ENDFOR
\vspace{0.6em}
\STATE \textit{// obtain posterior samples }
\STATE \textbf{return} $\Theta^{(N)} = \{\rvtheta_0^{(N,j)}\}_{j=1}^M \sim p(\rvtheta_0\mid \rvx^{(1:N)}_{\text{obs}})$
\end{algorithmic}
\end{algorithm}

\section{Numerical Experimental Details}
\label{app:experiment_details}

\subsection{Simulators}
\label{app:simulators}

\textbf{Two Moons} \citep{lueckmann2021benchmarking} is a widely used benchmark in SBI, designed as a
two-dimensional task that presents a posterior distribution with both global (bimodality) and local
(crescent shape) structure. For a given parameter vector $\rvtheta = (\evtheta_1, \evtheta_2) \in \R^2$,
the simulator generates data $\rvx \in \R^2$ through a random angle $a$, a random radius $r$, and an intermediate crescent point $\rvc$:
\begin{align*}
    a &\sim \gU(-\pi/2,\, \pi/2), \\
    r &\sim \gN(0.1,\; 0.01^2), \\
    \rvc &= (r\cos(a) + 0.25,\; r\sin(a)), \\
    \rvx &= \rvc + \left(\frac{-|\evtheta_1 + \evtheta_2|}{\sqrt{2}},\; \frac{-\evtheta_1 + \evtheta_2}{\sqrt{2}}\right).
\end{align*}
To obtain ground-truth posterior samples, we perform rejection sampling using the target prior as a
proposal distribution, with the constant $M$ set to the upper bound of the likelihood
$p(\rvx\mid\rvtheta)$ achieved at $r = 0.1$.

\textbf{Ornstein-Uhlenbeck Process (OUP)} \citep{uhlenbeck1930theory} is a well-established stochastic
process frequently applied in financial mathematics and evolutionary biology for modeling mean-reverting
dynamics. For a given parameter vector $\rvtheta = (\evtheta_1, \evtheta_2) \in \R^2$, the model is defined as:
\begin{equation*}
     y_{k+1} = y_k + \evtheta_1\left[\exp(\evtheta_2) - y_k\right]\Delta t + 0.5\,w_k,
    \quad \text{for } k = 0, \ldots, K-1,
\end{equation*}
where we set $K = 25$, $\Delta t = 0.2$, $y_0 = 10$, and $w_k \sim \gN(0, \Delta t)$ independently across $k$.
Simulated data are normalized using z-scoring at training time. Although the OUP admits a tractable likelihood, we treat the simulator as a black box to remain consistent with the SBI setting, and use a trained score-based NPE model (trained on one million simulations) as a surrogate
ground-truth posterior.

\textbf{Turin} \citep{turin1972statistical} is a widely used time-series model for simulating radio
wave propagation. This model generates high-dimensional, complex-valued time-series data and is
governed by four key parameters: $g_0$ determines the reverberation gain, $T_{\text{rev}}$ controls the
reverberation time, $\lambda_0$ defines the arrival rate of the point process, and $\sigma^2$
represents the noise variance.

The model assumes a frequency bandwidth of $B = 0.5$ GHz and simulates the transfer function $H_k$ at $N_s = 101$ evenly spaced frequency points. The number of scattering paths is sampled as $N_{\text{paths}} \sim \mathrm{Poisson}(\lambda_0 \cdot \tau_{\max}) + 1$ with $\tau_{\max} = 10^{-7}$\,s. Time delays are drawn as $\tau_l \sim \gU(0, \tau_{\max})$, and the complex gains are
$\alpha_l \sim \gCN\!\left(0,\; g_0 \exp(-\tau_l / T_{\text{rev}}) / \lambda_0\right)$. The transfer function is:
\begin{equation*}
    H_k = \sum_{l=1}^{N_{\text{paths}}} \alpha_l \exp(-\mathrm{i}\, 2\pi\, \Delta f\, k\, \tau_l), \quad k = 0, \ldots, N_s - 1,
\end{equation*}
where $\mathrm{i} = \sqrt{-1}$ is the imaginary unit and $\Delta f = B / (N_s - 1)$. The observed transfer function at the $k$-th frequency point is $Y_k = H_k + W_k$ with $W_k \sim \gCN(0, \sigma^2)$. The time-domain signal is obtained via an inverse Fourier transform, $\tilde{\mathbf{y}} = \mathrm{IFFT}(\mathbf{Y})$ where $\mathbf{Y} = (Y_0, \ldots, Y_{N_s-1})^\top$, and the real-valued output is:
\begin{equation*}
    y_k = 10 \log_{10}(|\tilde{y}_k|^2 + \varepsilon), \quad k = 0, \ldots, N_s - 1,
\end{equation*}
where $\varepsilon = 10^{-12}$ is a small constant added for numerical stability. Parameters are sampled directly in the parameter space and passed to the simulator without further rescaling.
Simulated outputs are normalized via z-scoring at training time. As the Turin likelihood is implicit,
we use a score-based NPE model trained on one million simulations as a surrogate ground truth.

\textbf{Gaussian Linear} \citep{lueckmann2021benchmarking} is a standard SBI benchmark task used to
infer the mean of a multivariate Gaussian distribution when the covariance is fixed. In this model,
both the parameters $\rvtheta$ and the data $\rvx$ are 6-dimensional vectors.
The simulator is defined as:
\begin{equation*}
    \rvx \mid \rvtheta \sim \gN(\rvtheta,\, \boldsymbol{\Sigma}_s), \quad \boldsymbol{\Sigma}_s = 0.1\, \rmI_6.
\end{equation*}
In all experiments, the test-time priors are constructed as Gaussian or mixture-of-Gaussian distributions, so each new prior has a closed-form posterior, which we use as the
ground-truth posterior.

\textbf{Simple Likelihood Complex Posterior (SLCP)} \citep{papamakarios2019sequential} is a benchmark
task in SBI that features a tractable likelihood but a complex, multi-modal posterior. For a given
parameter vector $\rvtheta = (\evtheta_1, \evtheta_2, \evtheta_3, \evtheta_4, \evtheta_5) \in \R^5$,
the simulator generates data $\rvx \in \R^8$ according to the following process:
\begin{align*}
    \boldsymbol{\mu} &= (\evtheta_1,\, \evtheta_2)^\top, \\
    \rho &= \tanh(\evtheta_5), \\
    \boldsymbol{\Sigma} &= \begin{pmatrix} \evtheta_3^2 & \rho\, \evtheta_3 \evtheta_4 \\ \rho\, \evtheta_3 \evtheta_4 & \evtheta_4^2 \end{pmatrix}, \\
    \rvx_i \mid \rvtheta &\sim \gN(\boldsymbol{\mu},\, \boldsymbol{\Sigma}), \quad i = 1, \ldots, 4, \\
    \rvx &= \big(\rvx_1^\top,\, \rvx_2^\top,\, \rvx_3^\top,\, \rvx_4^\top\big)^\top.
\end{align*}
The covariance $\boldsymbol{\Sigma}$ depends on $(\theta_3, \theta_4, \theta_5)$ only through
$\theta_3^2$, $\theta_4^2$ and the off-diagonal term $\tanh(\theta_5)\,\theta_3\theta_4$.
Since $\tanh$ is odd, the likelihood is invariant under the sign flips
$(\theta_3, \theta_4, \theta_5) \mapsto (-\theta_3, \theta_4, -\theta_5)$,
$(\theta_3, -\theta_4, -\theta_5)$ and $(-\theta_3, -\theta_4, \theta_5)$.
Both target priors are symmetric about zero in $\theta_3$, $\theta_4$ and $\theta_5$,
so the posterior inherits these symmetries and has four symmetric modes, yielding
the complex posterior structure. To obtain ground-truth posterior samples, we run
NUTS targeting the posterior under the target prior. Since NUTS explores only one
mode at a time, we run the chain for $n/4$ samples, generate the remaining three
modes by applying the three sign flips above, and then randomly subsample $n$ samples.

\subsection{Prior Specifications}
\label{app:prior}

We specify the training prior $p(\rvtheta)$ and the target prior $q(\rvtheta)$ for each simulator. Below, $\mathbf{1}_d$ denotes the all-ones vector, $\mathbf{1}_{d \times d}$ the all-ones matrix, $\mathbf{0}$ the zero vector, and $\rmI_d$ the $d \times d$ identity matrix. The mixture covariance $\mathbf{\Sigma}_{\text{mix}}$ used below is set to
$\mathbf{\Sigma}_{\text{mix}} = -0.02\,\mathbf{1}_{6\times6} + 0.22\,\mathbf{I}_6$.

\begin{description}
    \item[Two Moons] ($\rvtheta \in \R^2$)\hfill \\
    \textbf{ID:} $\begin{aligned}[t] 
        p(\rvtheta) &= \gN\!\left(\mathbf{0},\, \begin{pmatrix} 1.5 & 0.6 \\ 0.6 & 1.5 \end{pmatrix}\right), \\
        q(\rvtheta) &= \gN\!\left((0.3, 0.5)^\top,\, \operatorname{diag}(0.4^2, 0.3^2)\right) 
    \end{aligned}$ \\[0.3em]
    \textbf{OOD:} $\begin{aligned}[t] 
        p(\rvtheta) &= \gN\!\left((-0.3, 0.2)^\top,\, \begin{pmatrix} 1.4 & 0.6 \\ 0.6 & 1.4 \end{pmatrix}\right), \\
        q(\rvtheta) &= \gN\!\left(\mathbf{0},\, \begin{pmatrix} 1.2 & -0.2 \\ -0.2 & 1.2 \end{pmatrix}\right)
    \end{aligned}$

    \item[Gaussian Linear] ($\rvtheta \in \R^6$)\hfill \\
    \textbf{ID:} $\begin{aligned}[t] 
        p(\rvtheta) &= \gN(\mathbf{0},\, 0.2\, \mathbf{1}_{6\times 6} + 0.8\, \rmI_6), \\
        q(\rvtheta) &= 0.5\, \gN(0.6\, \mathbf{1}_6,\, \boldsymbol{\Sigma}_{\text{mix}}) 
                       + 0.5\, \gN(-0.6\, \mathbf{1}_6,\, \boldsymbol{\Sigma}_{\text{mix}})
    \end{aligned}$ \\[0.3em]
    \textbf{OOD:} $\begin{aligned}[t] 
        p(\rvtheta) &= \gN(\mathbf{0},\, -0.1\, \mathbf{1}_{6\times 6} + 1.1\, \rmI_6), \\
        q(\rvtheta) &= \gN(1.1\, \mathbf{1}_6,\, 0.2\, \mathbf{1}_{6\times 6} + 0.8\, \rmI_6)
    \end{aligned}$

    \item[SLCP] ($\rvtheta \in \R^5$)\hfill \\
    \textbf{ID:} $\begin{aligned}[t] 
        p(\rvtheta) &= \gN(\mathbf{0},\, \operatorname{diag}(1^2, 2^2, 1^2, 1^2, 1^2)), \\
        q(\rvtheta) &= \gN((1, 0.5, 0, 0, 0)^\top,\, \operatorname{diag}(0.5^2, 1^2, 1^2, 0.5^2, 1^2))
    \end{aligned}$ \\[0.3em]
    \textbf{OOD:} $\begin{aligned}[t] 
        p(\rvtheta) &= \gN((0, 2, 0, 0, 0)^\top,\, \operatorname{diag}(0.5^2, 1^2, 0.6^2, 1^2, 0.8^2)), \\
        q(\rvtheta) &= \gN((1, 0.5, 0, 0, 0)^\top,\, \operatorname{diag}(0.5^2, 1^2, 1^2, 0.5^2, 1^2))
    \end{aligned}$

    \item[OUP] ($\rvtheta \in \R^2$)\hfill \\
    \textbf{ID:} $\begin{aligned}[t] 
        p(\rvtheta) &= \gN((2, 0)^\top,\, 0.5\, \rmI_2), \\
        q(\rvtheta) &= \gN((1.8, 0)^\top,\, 0.3\, \rmI_2)
    \end{aligned}$ \\[0.3em]
    \textbf{OOD:} $\begin{aligned}[t] 
        p(\rvtheta) &= \gN((1.5, 1.5)^\top,\, 0.5\, \rmI_2), \\
        q(\rvtheta) &= \gN((1.8, 0)^\top,\, 0.5\, \rmI_2)
    \end{aligned}$

    \item[Turin] ($\rvtheta \in \R^4$)\hfill \\
    \textbf{ID:} $\begin{aligned}[t] 
        p(\rvtheta) &= \gN((0.64, 0.55, 0.5, 0.6)^\top,\, \operatorname{diag}(0.05^2, 0.08^2, 0.07^2, 0.1^2)), \\
        q(\rvtheta) &= \gN((0.6, 0.5, 0.5, 0.5)^\top,\, \operatorname{diag}(0.02^2, 0.05^2, 0.05^2, 0.06^2))
    \end{aligned}$ \\[0.3em]
    \textbf{OOD:} $\begin{aligned}[t] 
        p(\rvtheta) &= \gN((0.7, 0.6, 0.65, 0.35)^\top,\, \operatorname{diag}(0.04^2, 0.07^2, 0.06^2, 0.06^2)), \\
        q(\rvtheta) &= \gN((0.6, 0.5, 0.5, 0.5)^\top,\, \operatorname{diag}(0.02^2, 0.05^2, 0.05^2, 0.06^2))
    \end{aligned}$
\end{description}

\subsection{Training Setup}
We employ a two-stage training procedure. First, the base Simformer model is trained on 50,000 simulator-generated samples. Second, for our proposed NRE-Based and KLIEP-Based Guidance, we train a lightweight log-ratio model using an augmented dataset of 470,000 pairs. This dataset comprises the original 50,000 samples supplemented by 420,000 posterior samples generated by the pretrained Simformer. The baseline PriorGuide requires no retraining of the base diffusion model, but fits a Gaussian mixture model to the prior ratio for each target prior.

\paragraph{Simformer}
We adopt a similar setup as the Simformer paper \citep{gloeckler2024all}, using the Variance Preserving Stochastic Differential Equation (VP-SDE) with a linear noise schedule parameterized by $\beta_{\min} = 0.1$ and $\beta_{\max} = 10.0$, running over the time interval $t \in [10^{-5}, 1]$. We use a transformer configuration with 6 layers, 4 attention heads (size 10), a token dimension of 40, and a 128-dimensional Gaussian Fourier embedding for diffusion time. MLP blocks use a widening factor of 3. In all experiments, the condition mask is sampled per batch using a structured random scheme, uniformly selecting one of the following: the joint mask, the posterior mask, the likelihood mask, or one of two independently sampled random masks. We use the same setup for all simulators. Models are trained with a batch size of 1,000 and an initial learning rate of $10^{-3}$, decayed linearly to $1 \times 10^{-6}$, combined with adaptive gradient clipping (maximum norm 10.0) and the Adam optimizer \citep{kingma2015adam}. The number of training iterations is constrained to a minimum of 5,000 and a maximum of 50,000.

\paragraph{DRE-Based Guidance (NRE and KLIEP)}
Both methods share the same lightweight log-ratio model architecture: a 4-layer MLP with hidden dimension 256, SiLU activations, and a 4-dimensional Gaussian Fourier embedding for diffusion time. The network $a_\psi(\evtheta_t, x, t)$ takes as input the noisy parameter state $\evtheta_t$, the observation $x$, and the diffusion time $t$, and outputs a scalar estimate of the diffusion-space joint log-ratio $\log q(\evtheta_t, x)/p(\evtheta_t, x)$.

The training dataset consists of the original 50,000 simulated samples from the reference prior, augmented with 420,000 additional $(\evtheta, x)$ pairs generated by the pretrained Simformer model --- for each of 420 distinct observations $x$ drawn from the base training set, 1,000 posterior samples $\evtheta \sim p(\evtheta \mid x)$ are obtained via the base diffusion model, yielding 470,000 training samples in total.

Both methods are trained using AdamW with a batch size of 4,096 and a learning rate of $3 \times 10^{-4}$ for 10,000 iterations.

\paragraph{PriorGuide}
We represent the prior ratio $\log q(\evtheta)/p(\evtheta)$ with a Gaussian mixture model (GMM) of $K=50$ diagonal-covariance components, fitted by minimizing the mean squared error between its log-density and the analytic log-ratio evaluated at samples from the target prior $q(\evtheta)$. The GMM is trained using Adam with a learning rate of $10^{-3}$, a batch size of 200, and gradient norm clipping (maximum norm 1.0), for 20,000 iterations.

\paragraph{WSM} 
WSM uses the same architecture, dataset, and optimizer settings as the DRE-based methods (4-layer MLP, hidden dimension 256, SiLU activations, 4-dimensional Gaussian Fourier time embedding; AdamW, batch size 4{,}096, learning rate $3\times10^{-4}$), and is trained for 10,000 iterations.

\subsection{Inference Setup}
All methods utilize the same reverse diffusion sampler: a discretization of the reverse VP-SDE with $T=25$ steps over $t \in [t_{\min}, t_{\max}] = [10^{-5}, 1]$. We employ a power-law time schedule $t_i = \bigl(t_{\min} + \tfrac{i}{T}(t_{\max}-t_{\min})\bigr)^{\rho}$ with $\rho=2.0$. At test time, we draw 5,000 posterior samples per observation for all evaluated methods. Evaluation uses 10 held-out observations with fixed random seeds, disjoint from the 420 observations used for data augmentation during log-ratio model training.

\paragraph{DRE-Based Guidance (NRE and KLIEP)}
At each reverse diffusion step, the base score is additively corrected by the gradient of the learned log-ratio model.
Notably, our methods require no additional correction steps, maintaining high sampling efficiency.

\paragraph{PriorGuide}
We apply $4$ Langevin correction steps after each reverse diffusion step, with a step size $\epsilon_t = \tfrac{1}{4} \beta(t) |\Delta t|$, where $\beta(t)$ is the drift coefficient of the VP-SDE at time $t$. 

\paragraph{WSM}
As with the DRE-based methods, the base score is additively corrected at each reverse diffusion step using the directly learned correction. No Langevin correction steps or additional gradient evaluations are required.

\subsection{Hardware}
\label{app:hardware-numerical}
All numerical experiments described in this appendix were conducted on a single NVIDIA GH200 120GB Superchip (96 GB GPU memory). Inference, which involves drawing 5,000 posterior samples, completes within seconds per observation.

\section{Additional Numerical Experimental Results}
\label{app:experiment_res}

\begin{table}[H]
\centering
\footnotesize
\setlength{\tabcolsep}{3pt}
\caption{\textbf{Posterior inference performance ($\evtheta$) under prior shift.} Mean (standard dev.) over 10 simulated observations, where each observation's value is averaged over 3 independent runs. Bold indicates the best result or one not significantly worse than the best (one-sided Wilcoxon signed-rank test paired by observation, $p > 0.05$). Gray rows are shown for reference only and are excluded from the comparison. All evaluation metrics are computed using $5{,}000$ posterior samples per method run: C2ST, MMD, and $W_2$ are calculated between these $5{,}000$ generated samples and $5{,}000$ reference (ground-truth) posterior samples, while RMSE is computed between the $5{,}000$ generated samples and the true parameters $\evtheta^*$.}
\label{tab:posterior}
\resizebox{\textwidth}{!}{%
\begin{tabular}{llcccccccc}
\toprule
 & & \multicolumn{4}{c}{ID} & \multicolumn{4}{c}{OOD} \\
\cmidrule(lr){3-6} \cmidrule(lr){7-10}
 & & C2ST & MMD & WS-2 & RMSE & C2ST & MMD & WS-2 & RMSE \\
\midrule
\multirow{6}{*}{Two Moons}
 & Simformer       & 0.79(0.02)                & 0.34(0.06)                & 0.68(0.43)                & 0.71(0.42)                & \textbf{0.69(0.06)}       & \textbf{0.01(0.01)}       & 0.25(0.24)                & \textbf{1.04(0.89)}       \\
 & PriorGuide      & \textbf{0.68(0.01)}       & \textbf{0.01(0.00)}       & \textbf{0.15(0.10)}       & \textbf{0.20(0.08)}       & 0.71(0.06)                & \textbf{0.00(0.01)}       & \textbf{0.14(0.12)}       & \textbf{1.03(0.86)}       \\
 & WSM             & 0.91(0.02)                & 0.28(0.06)                & 4.79(3.63)                & 5.01(3.70)                & 0.86(0.03)                & 0.03(0.03)                & 0.34(0.11)                & 1.10(0.78)                \\
 & Ours (NRE)    & \textbf{0.69(0.01)}       & 0.02(0.00)                & \textbf{0.16(0.08)}       & \textbf{0.21(0.06)}       & 0.72(0.06)                & \textbf{0.00(0.01)}       & \textbf{0.15(0.11)}       & \textbf{1.03(0.86)}       \\
 & Ours (KLIEP)    & 0.69(0.01)                & 0.01(0.00)                & \textbf{0.15(0.08)}       & \textbf{0.20(0.06)}       & 0.72(0.06)                & 0.01(0.01)                & 0.20(0.16)                & \textbf{1.03(0.86)}       \\
 & \textcolor{gray}{Simformer+SIR} & \textcolor{gray}{0.85(0.01)} & \textcolor{gray}{0.01(0.00)} & \textcolor{gray}{0.01(0.01)} & \textcolor{gray}{0.09(0.02)} & \textcolor{gray}{0.80(0.04)} & \textcolor{gray}{0.01(0.01)} & \textcolor{gray}{0.17(0.12)} & \textcolor{gray}{1.03(0.86)} \\
\midrule
\multirow{6}{*}{SLCP}
 & Simformer       & 0.79(0.04)                & 0.03(0.02)                & 0.21(0.08)                & 1.12(0.32)                & 0.83(0.06)                & 0.08(0.09)                & 0.31(0.18)                & 1.16(0.31)                \\
 & PriorGuide      & 0.78(0.05)                & \textbf{0.03(0.03)}       & \textbf{0.16(0.06)}       & \textbf{1.05(0.32)}       & 0.76(0.06)                & \textbf{0.04(0.04)}       & \textbf{0.17(0.07)}       & \textbf{1.04(0.31)}       \\
 & WSM             & 0.79(0.04)                & \textbf{0.02(0.02)}       & 0.20(0.07)                & 1.11(0.31)                & 0.78(0.06)                & 0.05(0.05)                & 4.51(13.31)               & 8.90(24.45)               \\
 & Ours (NRE)    & 0.77(0.05)                & \textbf{0.02(0.03)}       & \textbf{0.16(0.06)}       & 1.08(0.31)                & \textbf{0.74(0.05)}       & \textbf{0.02(0.02)}       & \textbf{0.15(0.07)}       & \textbf{1.08(0.30)}       \\
 & Ours (KLIEP)    & \textbf{0.76(0.05)}       & \textbf{0.02(0.02)}       & \textbf{0.16(0.06)}       & 1.08(0.31)                & 0.75(0.06)                & 0.03(0.04)                & 0.19(0.11)                & 1.10(0.30)                \\
 & \textcolor{gray}{Simformer+SIR} & \textcolor{gray}{0.87(0.02)} & \textcolor{gray}{0.02(0.03)} & \textcolor{gray}{0.13(0.06)} & \textcolor{gray}{1.06(0.32)} & \textcolor{gray}{0.91(0.04)} & \textcolor{gray}{0.02(0.02)} & \textcolor{gray}{0.14(0.09)} & \textcolor{gray}{1.05(0.31)} \\
\midrule
\multirow{7}{*}{Gaussian Linear 6D}
 & Simformer       & 0.73(0.05)                & 0.09(0.04)                & 0.13(0.03)                & 0.43(0.07)                & 0.70(0.06)                & 0.08(0.03)                & 0.10(0.03)                & \textbf{0.34(0.06)}       \\
 & PriorGuide      & 0.55(0.01)                & 0.01(0.00)                & 0.04(0.01)                & \textbf{0.40(0.05)}       & 0.69(0.08)                & 0.06(0.05)                & 0.09(0.04)                & 0.37(0.07)                \\
 & WSM             & 0.60(0.04)                & 0.02(0.02)                & 0.06(0.03)                & 0.41(0.06)                & 0.67(0.06)                & 0.06(0.03)                & 0.11(0.05)                & 0.37(0.07)                \\
 & Ours (NRE)    & \textbf{0.54(0.01)}       & \textbf{0.00(0.00)}       & \textbf{0.03(0.00)}       & \textbf{0.39(0.04)}       & \textbf{0.58(0.05)}       & \textbf{0.02(0.01)}       & \textbf{0.04(0.02)}       & \textbf{0.34(0.05)}       \\
 & Ours (KLIEP)    & \textbf{0.54(0.01)}       & \textbf{0.00(0.00)}       & \textbf{0.03(0.00)}       & \textbf{0.39(0.04)}       & 0.60(0.05)                & 0.02(0.01)                & 0.05(0.02)                & \textbf{0.34(0.05)}       \\
 & \textcolor{gray}{Simformer+SIR} & \textcolor{gray}{0.85(0.04)} & \textcolor{gray}{0.00(0.00)} & \textcolor{gray}{0.02(0.00)} & \textcolor{gray}{0.38(0.04)} & \textcolor{gray}{0.83(0.05)} & \textcolor{gray}{0.01(0.01)} & \textcolor{gray}{0.03(0.01)} & \textcolor{gray}{0.33(0.05)} \\
 & \textcolor{gray}{Analytic} & \textcolor{gray}{0.51(0.00)} & \textcolor{gray}{0.00(0.00)} & \textcolor{gray}{0.02(0.00)} & \textcolor{gray}{0.39(0.04)} & \textcolor{gray}{0.55(0.03)} & \textcolor{gray}{0.01(0.01)} & \textcolor{gray}{0.03(0.01)} & \textcolor{gray}{0.34(0.04)} \\
\midrule
\multirow{6}{*}{OUP}
 & Simformer       & 0.52(0.02)                & 0.01(0.02)                & 0.03(0.03)                & \textbf{0.20(0.08)}       & 0.60(0.08)                & 0.10(0.12)                & 0.09(0.08)                & \textbf{0.25(0.13)}       \\
 & PriorGuide      & 0.53(0.03)                & 0.02(0.03)                & 0.04(0.03)                & \textbf{0.20(0.06)}       & 0.65(0.10)                & 0.18(0.23)                & 0.25(0.27)                & 0.42(0.22)                \\
 & WSM             & 0.55(0.06)                & 0.01(0.01)                & 0.04(0.03)                & 0.22(0.07)                & \textbf{0.56(0.06)}       & \textbf{0.04(0.04)}       & \textbf{0.30(0.74)}       & \textbf{0.54(0.96)}       \\
 & Ours (NRE)    & \textbf{0.51(0.01)}       & \textbf{0.00(0.00)}       & \textbf{0.01(0.01)}       & \textbf{0.20(0.06)}       & \textbf{0.54(0.05)}       & \textbf{0.02(0.04)}       & \textbf{0.05(0.05)}       & \textbf{0.24(0.09)}       \\
 & Ours (KLIEP)    & \textbf{0.51(0.01)}       & \textbf{0.00(0.00)}       & \textbf{0.01(0.01)}       & \textbf{0.20(0.06)}       & \textbf{0.54(0.04)}       & \textbf{0.02(0.03)}       & \textbf{0.05(0.04)}       & \textbf{0.24(0.09)}       \\
 & \textcolor{gray}{Simformer+SIR} & \textcolor{gray}{0.68(0.01)} & \textcolor{gray}{0.00(0.00)} & \textcolor{gray}{0.01(0.01)} & \textcolor{gray}{0.19(0.06)} & \textcolor{gray}{0.72(0.05)} & \textcolor{gray}{0.02(0.04)} & \textcolor{gray}{0.05(0.05)} & \textcolor{gray}{0.24(0.10)} \\
\midrule
\multirow{6}{*}{Turin}
 & Simformer       & 0.84(0.02)                & 0.18(0.04)                & 0.04(0.01)                & 0.09(0.01)                & 0.99(0.00)                & 0.82(0.02)                & 0.10(0.00)                & 0.13(0.03)                \\
 & PriorGuide      & \textbf{0.53(0.02)}       & \textbf{0.01(0.00)}       & \textbf{0.01(0.00)}       & \textbf{0.07(0.02)}       & \textbf{0.53(0.02)}       & \textbf{0.01(0.01)}       & \textbf{0.01(0.00)}       & \textbf{0.07(0.02)}       \\
 & WSM             & 0.61(0.03)                & 0.03(0.01)                & 0.01(0.00)                & 0.07(0.01)                & 0.82(0.04)                & 0.22(0.09)                & 0.05(0.01)                & 0.10(0.02)                \\
 & Ours (NRE)    & 0.57(0.03)                & 0.02(0.01)                & 0.01(0.00)                & \textbf{0.07(0.02)}       & 0.71(0.03)                & 0.08(0.03)                & 0.02(0.00)                & 0.08(0.02)                \\
 & Ours (KLIEP)    & 0.55(0.02)                & 0.01(0.01)                & 0.01(0.00)                & \textbf{0.07(0.02)}       & 0.73(0.04)                & 0.10(0.05)                & 0.03(0.01)                & 0.08(0.02)                \\
 & \textcolor{gray}{Simformer+SIR} & \textcolor{gray}{0.91(0.01)} & \textcolor{gray}{0.00(0.00)} & \textcolor{gray}{0.00(0.00)} & \textcolor{gray}{0.07(0.02)} & \textcolor{gray}{1.00(0.00)} & \textcolor{gray}{0.26(0.04)} & \textcolor{gray}{0.03(0.00)} & \textcolor{gray}{0.07(0.02)} \\
\bottomrule
\end{tabular}%
}
\end{table}

\vspace{-8pt}

\begin{table}[h]
\centering 
\footnotesize
\setlength{\tabcolsep}{3pt}
\caption{\textbf{Posterior predictive RMSE.} Mean (standard dev.) over 10 simulated observations, where each observation's value is averaged over 3 independent runs. Bold indicates the best result or one not significantly worse than the best (one-sided Wilcoxon signed-rank test paired by observation, $p > 0.05$). Gray rows are shown for reference only and are excluded from the comparison.}
\label{tab:posterior_predictive_results}
\begin{tabular}{llcc}
\toprule
 & & \multicolumn{2}{c}{RMSE} \\
\cmidrule(lr){3-4}
 & & ID & OOD \\
\midrule
\multirow{6}{*}{OUP}
 & Simformer       & \textbf{0.54(0.10)}       & \textbf{0.59(0.15)}       \\
 & PriorGuide      & 0.82(0.57)                & 1.18(1.06)                \\
 & WSM             & 0.55(0.10)                & \textbf{0.66(0.24)}       \\
 & Ours (NRE)    & \textbf{0.56(0.12)}       & \textbf{0.60(0.13)}       \\
 & Ours (KLIEP)    & 0.56(0.12)                & \textbf{0.61(0.15)}       \\
 & \textcolor{gray}{Simformer+SIR} & \textcolor{gray}{0.56(0.13)} & \textcolor{gray}{0.59(0.18)} \\
\midrule
\multirow{6}{*}{Turin}
 & Simformer       & 8.46(0.75)                & 8.58(0.83)                \\
 & PriorGuide      & 8.54(0.82)                & \textbf{8.52(0.84)}       \\
 & WSM             & \textbf{8.27(0.79)}       & 10.84(1.86)               \\
 & Ours (NRE)    & 8.57(0.76)                & \textbf{8.52(0.79)}       \\
 & Ours (KLIEP)    & 8.58(0.76)                & 8.63(0.81)                \\
 & \textcolor{gray}{Simformer+SIR} & \textcolor{gray}{8.46(0.79)} & \textcolor{gray}{8.42(0.73)} \\
\bottomrule
\end{tabular}
\end{table}

\begin{figure}[H]
    \centering
    \includegraphics[width=1\linewidth]{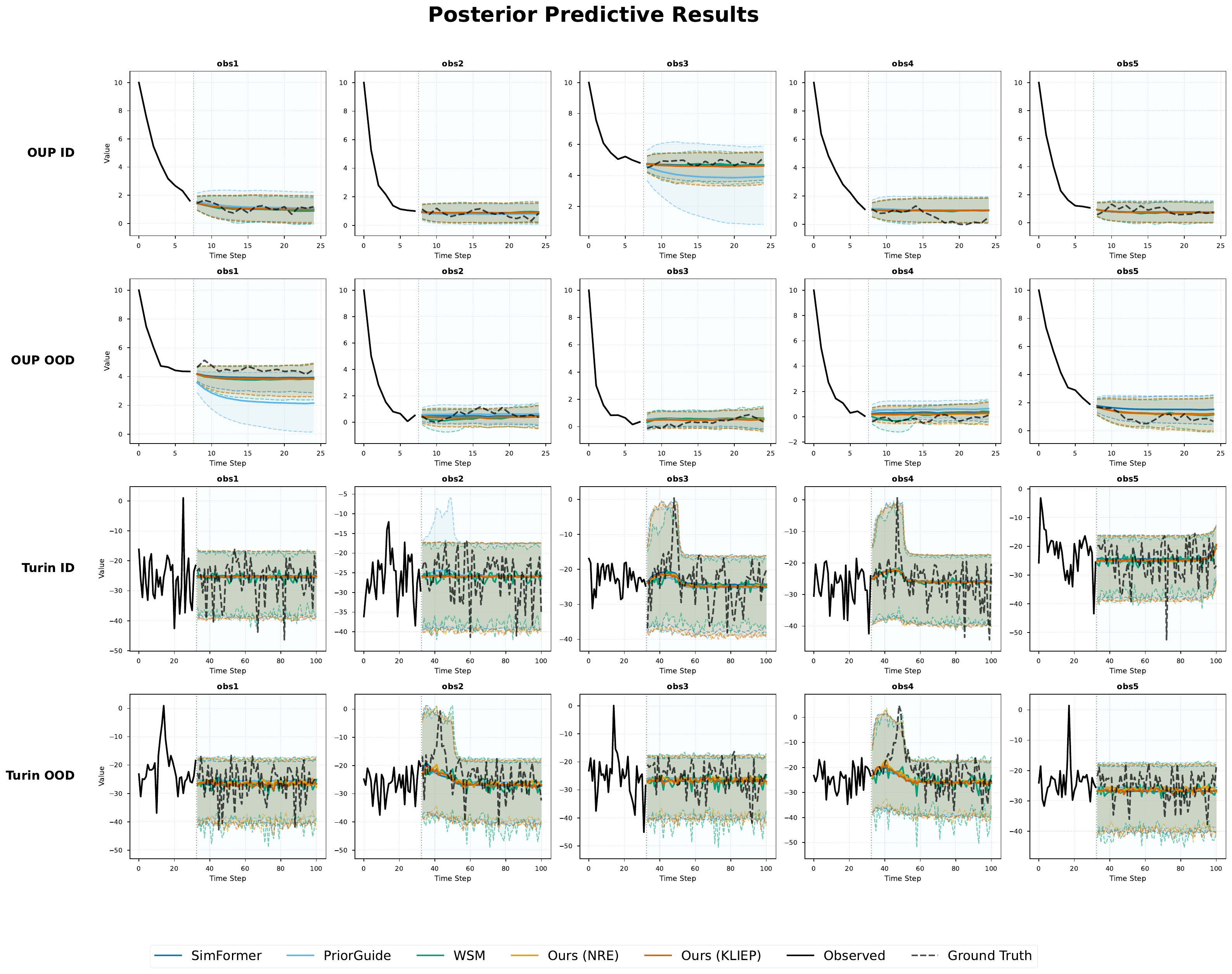}
    \caption{Posterior predictive results across all tasks and observation sets. Each row corresponds to a task setting (OUP ID, OUP OOD, Turin ID, Turin OOD), and each column shows one of the first five observations. Solid lines denote predictive means, while dashed lines and shaded regions indicate 95\% confidence intervals. Black curves represent observed data (solid) and ground truth (dashed). The vertical line marks the prediction boundary.}
    \label{fig:placeholder}
\end{figure}

\section{Failure Case Analysis}
\label{app:failure}
\subsection{Two Moons}
For each task, we draw $N$ samples $\rvtheta^{(i)} \sim p(\rvtheta \mid \mathbf{x})$
from the base model's posterior and compute the effective sample fraction of the
prior-ratio weights $w_i = q(\rvtheta^{(i)})/p(\rvtheta^{(i)})$, given by
\[
\mathrm{ESS}/N = \Big(\sum_{i=1}^{N} w_i\Big)^2 \Big/ \Big(N \sum_{i=1}^{N} w_i^2\Big) \in (0,1].
\] 
This quantity is close to $1$ when the ratio is nearly constant over the posterior support and close to $0$ when few samples dominate. Two Moons is the task with $\mathrm{ESS}/N$ closest to $1$, at $0.990$ (see Table~\ref{tab:essn}): the two priors still differ substantially over most of the parameter space, but the posterior concentrates in a region where their ratio is nearly flat. In this regime, the exact guidance term is close to zero, so the guidance methods mainly reflect their own approximation error rather than genuine guidance signal. Since $\mathrm{ESS}/N$ requires only base samples and the prior ratio, it can be computed before adaptation as an inexpensive diagnostic of whether the prior update carries enough signal to justify the cost of guidance.
\begin{table}[H]
\centering
\caption{Effective sample fraction $\mathrm{ESS}/N$ of the prior-ratio weights
$w_i = q(\rvtheta^{(i)})/p(\rvtheta^{(i)})$ over base-model posterior samples.}
\label{tab:essn}
\begin{tabular}{lcc}
\toprule
& \multicolumn{2}{c}{$\mathrm{ESS}/N$} \\
\cmidrule(lr){2-3}
Task & ID & OOD \\
\midrule
Two Moons          & 0.913 & 0.990 \\
OUP                & 0.957 & 0.487 \\
SLCP               & 0.854 & 0.127 \\
Gaussian Linear 6D & 0.574 & 0.209 \\
Turin              & 0.642 & 0.006 \\
\bottomrule
\end{tabular}
\end{table}

\subsection{Turin}
PriorGuide outperforms our method on Turin. The DRE component of our approach operates on the joint space of parameters and observations, and DRE is known to degrade in high dimensions \citep{choi2022density}. Turin has the largest combined dimension $d_\theta + d_x = 105$, in our numerical experiments, which may explain the weaker performance. A standard remedy in SBI is to compress $x$ with a learned embedding network~\citep{radev2020bayesflow}; we leave this to future work.

\section{Experimental Details for Planet Light Curve Inference}
\label{sec:app_lightcurve}

\subsection{Data Acquisition and Preprocessing}

We utilize real-world photometric data from the TESS mission, accessed via the
\texttt{Lightkurve} library \citep{2018ascl.soft12013L}.
We evaluate on three exoplanet systems spanning a range of transit depths and
orbital configurations: \textbf{WASP-126\,b}, \textbf{WASP-39\,b},
and \textbf{GJ~436\,b}.
For each target, individual sector light curves are downloaded, sigma-clipped,
normalized, and detrended before being stitched into a single time series.
Transit periods and reference epochs are estimated via a Box Least-Squares (BLS)
periodogram \citep{kovacs2002box}.

For each predicted transit time, we open a wide search window ($1.5\times$ the
nominal half-width) around the BLS epoch and locate the actual flux minimum using
Savitzky--Golay smoothing followed by a sub-cadence parabolic vertex fit.
The final window is re-cropped to $\pm 0.5$\,days relative to the detected dip
center, ensuring that the transit is consistently aligned to $t = 0$ across all
windows---matching the convention used during simulator training.
Each segment is then resampled onto a uniform grid of $N = 128$ points via linear
interpolation. Windows failing quality checks (mean flux deviating from unity by
more than 5\%, minimum flux below 0.5, or a linear trend exceeding
$2 \times 10^{-4}$\,per cadence) are discarded.

After preprocessing and quality filtering, the final datasets contain
$N_{\mathrm{obs}}=5$ transit windows for WASP-39\,b, and
$N_{\mathrm{obs}}=7$ windows for both GJ~436\,b and WASP-126\,b.
Each window corresponds to one observation $\rvx^{(i)} \in \R^{128}$ used in
the sequential posterior update procedure.

\subsection{Light Curve Observation Data}
\label{app:lc_obs}
\begin{figure}[H]
    \centering
    \includegraphics[width=1\linewidth]{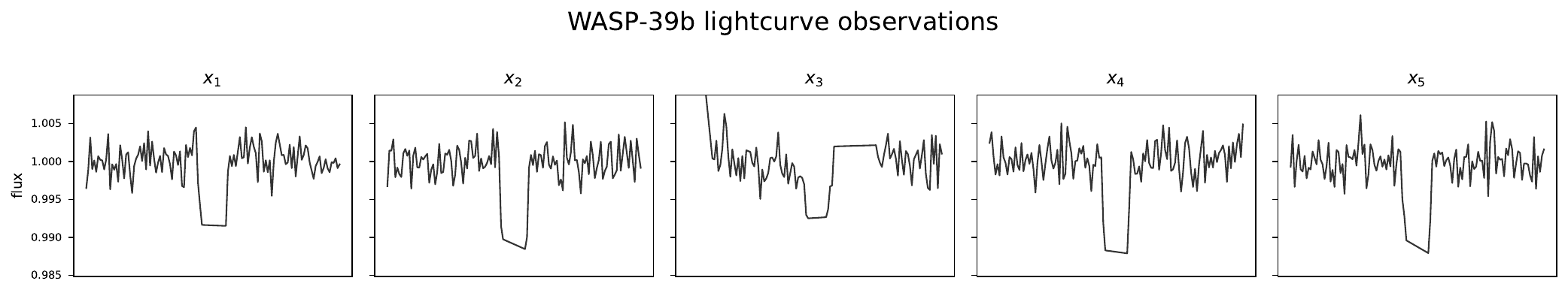}
    \caption{Example observation sequence for WASP-39\,b. Each window $\rvx^{(n)}_{\text{obs}}$ is a normalized light-curve segment
    centered on a transit event; the five segments $x_1,\dots,x_5$ are incorporated sequentially in the posterior updates shown in Figure~\ref{fig:lc_1}.}
    \label{fig:lc_obs}
\end{figure}

\subsection{Simulator and Forward Modeling}
Synthetic training samples are generated using \texttt{PyTransit} \citep{Parviainen2015}. The simulator implements a quadratic limb-darkening model to compute the relative flux. We assume a diagonal Gaussian reference prior $p(\evtheta) = \gN(\mu, \text{diag}(\sigma^2))$ over the 8-dimensional parameter space, as detailed in Table~\ref{tab:lc_priors}.

\begin{table}[H]
\centering
\caption{Training prior specifications for the light curve task. Parameters marked with $\star$ are the primary parameters of interest.}
\label{tab:lc_priors}
\begin{tabular}{ccllcc}
\toprule
Index &  & Parameter & Description & Mean ($\mu$) & Std ($\sigma$) \\
\midrule
0 & $\star$ & $k$ & Planet-to-star radius ratio & 0.10 & 0.04 \\
1 & $\star$ & $a$ & Scaled semi-major axis ($a/R_*$) & 10.0 & 3.0 \\
2 & $\star$ & $\mathrm{inc}$ & Orbital inclination (deg) & 87.5 & 2.0 \\
3 & $\star$ & $P$ & Orbital period (days) & 3.6 & 0.8 \\
4 &  & $t_0$ & Transit midpoint offset (days) & 0.0 & 0.05 \\
5 &  & $u_1$ & Limb-darkening coefficient 1 & 0.35 & 0.08 \\
6 &  & $u_2$ & Limb-darkening coefficient 2 & 0.18 & 0.06 \\
7 &  & $\sigma$ & Observational noise std & 0.001 & 0.0003 \\
\bottomrule
\end{tabular}
\end{table}

\subsection{Training and Guidance Protocol}
\paragraph{Base score model.}
The base score model is trained on $3 \times 10^5$ synthetic samples for 8{,}000 iterations using the Adam optimizer with a batch size of 2{,}048 and learning rate $10^{-3}$.
\paragraph{DRE for prior ratio.}
The prior ratio $r_0(\evtheta) \approx q(\evtheta)/p(\evtheta)$, used by both our method and the PriorGuide baseline, is estimated by a 3-layer MLP with 128 hidden units and SiLU activations, trained via a logistic DRE objective (AdamW, batch size 4{,}096, learning rate $10^{-3}$, 16{,}000 iterations).
\paragraph{Sequential adaptation (our method).}
At each transit window $n$, the prior-ratio DRE above is trained to obtain
\[
  \rho^{(n-1)}(\evtheta) \approx \log \frac{q^{(n-1)}(\evtheta)}{p(\evtheta)},
\]
using posterior samples of the $n-1$-th transit window ($q^{(n-1)}$) against the original synthetic prior $p(\evtheta)$. A DRE implemented as a 4-layer MLP with 256 hidden units and a 16-dimensional Fourier time embedding, is then trained for 26{,}000 iterations (batch size 4{,}096, learning rate $10^{-4}$) using NRE with prior ratio $r^{(n-1)}(\evtheta) = \exp(\rho^{(n-1)}(\evtheta))$.

\paragraph{Baseline: PriorGuide (Monte Carlo variant).}
Since the target prior at each step is available only as samples, we replace the GMM fit in the original PriorGuide with a Monte Carlo estimate of the guidance term. The prior-ratio DRE above provides $r_0(\evtheta) \approx  q(\evtheta)/p(\evtheta)$. At each reverse diffusion step $t$ ($T=25$ steps), $K=1{,}000$ samples are drawn from the Tweedie posterior $p(\evtheta_0 \mid \evtheta_t) \approx \gN(\hat\mu_0(t), \sigma_{\mathrm{post}}^2(t)\,I)$, with $\sigma_{\mathrm{post}}^2(t) = (\alpha_t^2/\sigma_t^2 + 1)^{-1}$, $\alpha_t = \sqrt{1-\sigma^2(t)}$, and $\sigma_t = \sigma(t)$. After each reverse diffusion step, $8$ Langevin correction steps are applied at the current noise level, with step size $\epsilon_t = \beta(t)\,|\Delta t| / 4$  where $\beta(t)$ is the drift coefficient of the VP-SDE at time $t$. 
\paragraph{Guided inference.}
For all methods, guided reverse-diffusion sampling uses $T = 25$ steps with a power-law time schedule ($\rho = 2.0$). We draw $10^5$ posterior samples per transit window. The transit midpoint offset $t_0$ is excluded from the sequential update, as each window is independently re-centered at the flux minimum.
\paragraph{Targets.}
We evaluate on three exoplanet systems observed by TESS: WASP-126\,b, WASP-39\,b, and GJ~436\,b, spanning a range of transit depths and orbital geometries.
\subsection{Hardware}
\label{app:hardware-lightcurve}
All experiments for the planet light-curve task were conducted on a single NVIDIA GH200 120GB Superchip (96 GB GPU memory), where posterior updates and guided sampling take approximately 10 minutes.

\section{Additional Planet Light Curve Inference Results}
\label{app:lc_res}

\begin{figure}[H]
    \centering
    \includegraphics[width=0.9\linewidth]{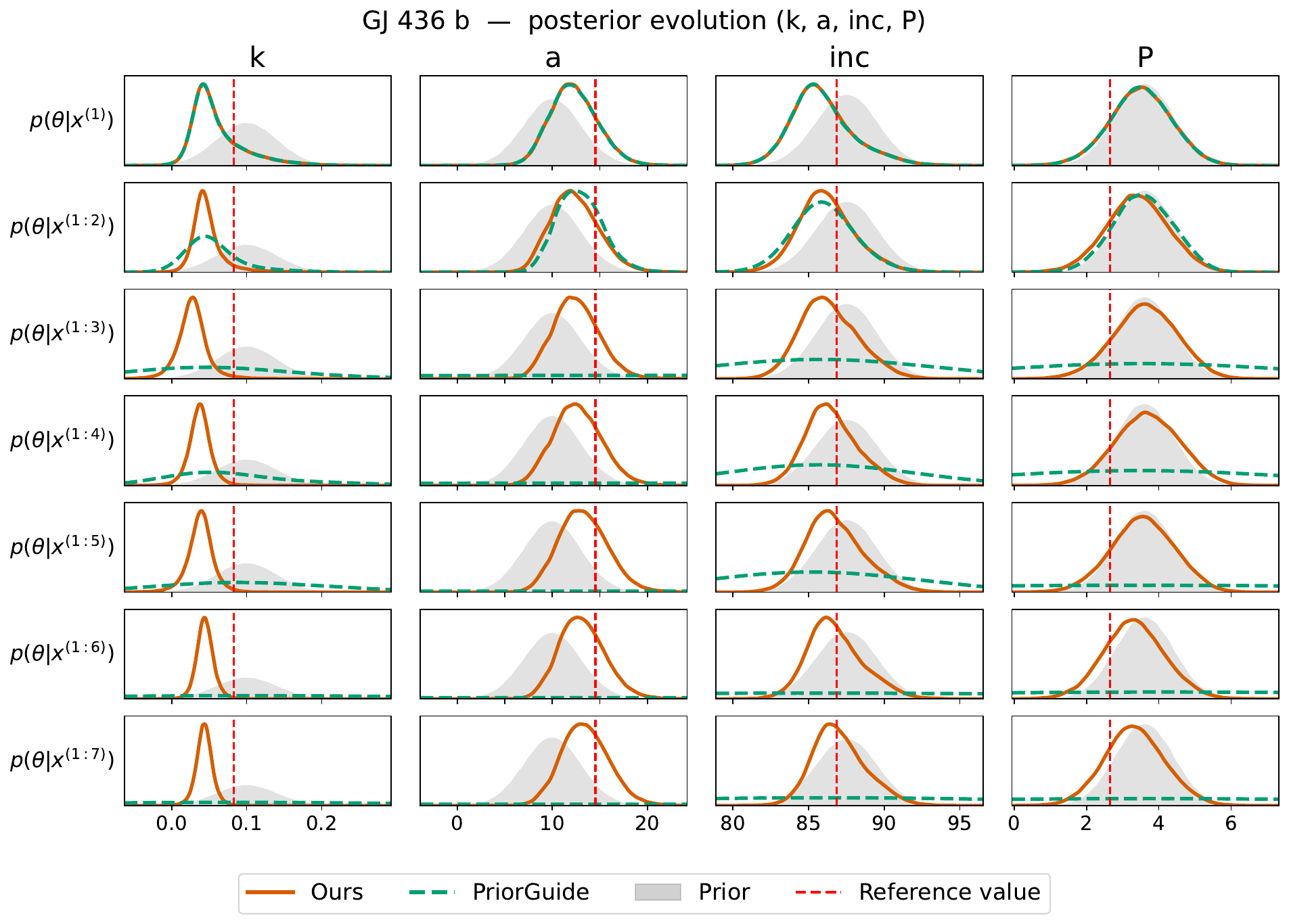}
    \caption{Sequential posterior evolution for the \textbf{GJ~436\,b} system (reference value from \citep{lanotte2014global}). Layout and notation follow Figure~\ref{fig:lc_1}.}
    \label{fig:lc_2}
\end{figure}

\begin{figure}[H]
    \centering
    \includegraphics[width=0.9\linewidth]{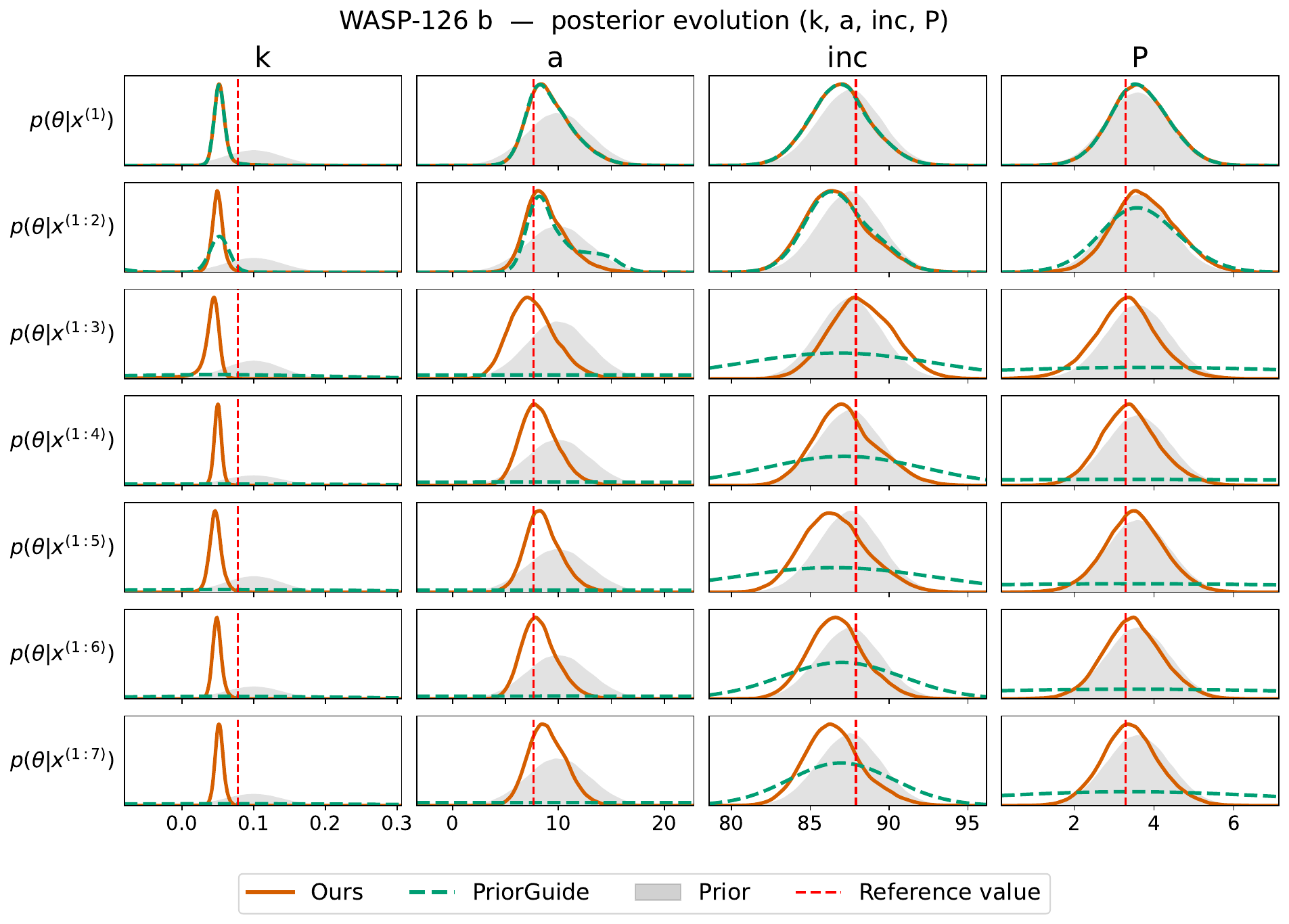}
    \caption{Sequential posterior evolution for the \textbf{WASP-126\,b} system (reference value from \citep{maxted2016five}). Layout and notation follow Figure~\ref{fig:lc_1}.}
    \label{fig:lc_3}
\end{figure}

\begin{figure}[H]
    \centering
    \includegraphics[width=1\linewidth]{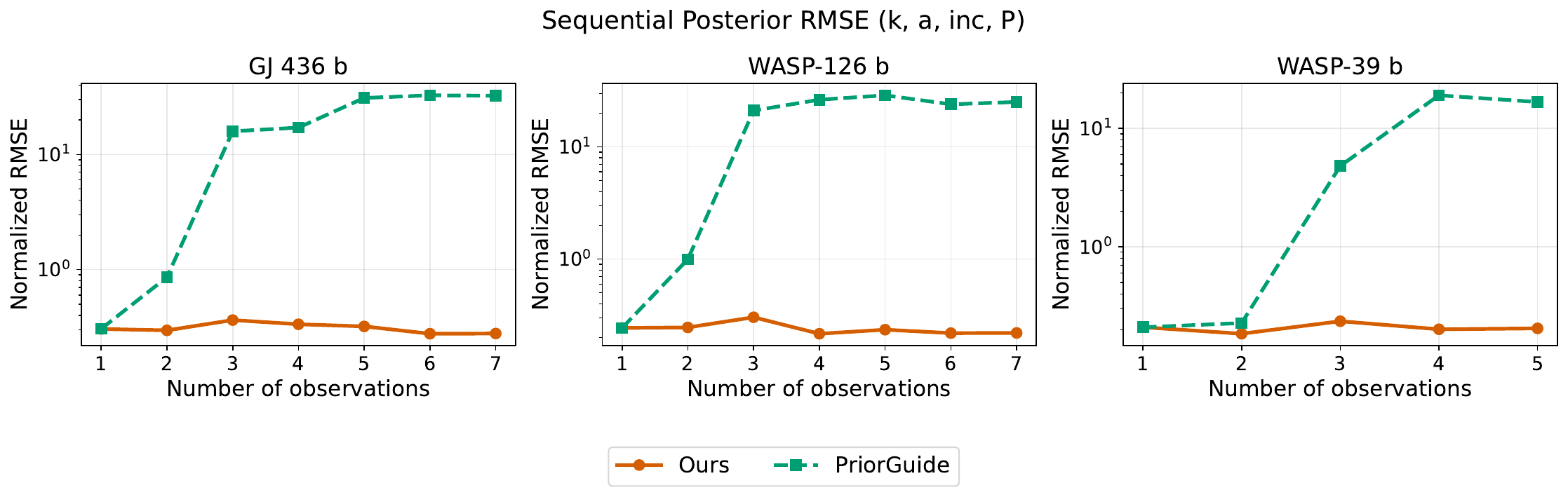}
    \caption{Normalized posterior RMSE across sequential observations for three exoplanet systems, computed using the transit parameters $k$, $a$, inc, and $P$ (normalized per parameter, then averaged). Our method maintains stable error across sequential updates, while PriorGuide exhibits progressively increasing error.}
    \label{fig:lc_rmse}
\end{figure}

\end{document}